\documentclass[twocolumn, aps]{revtex4-2}

\usepackage{amsmath}
\usepackage{amssymb}
\usepackage{amsthm}
\usepackage{bbm}
\usepackage{csquotes}
\usepackage{enumitem}
\usepackage{mathtools}
\usepackage[caption=false]{subfig} % caption=false to avoid clash with RevTeX and glossaries

\usepackage{tikz}
\usetikzlibrary{patterns,arrows.meta}

\usepackage[acronym]{glossaries}

\usepackage{hyperref}
\hypersetup{%
	colorlinks=true,
	linkcolor=red,
	urlcolor=red,
	citecolor=red,
	breaklinks=true,
}
\usepackage{xurl}

\usepackage{zref-clever}

\makeatletter
\pgfdeclarepatternformonly{ne dashed lines}%
{\pgfqpoint{-0.5pt}{-0.5pt}}%
{\pgfqpoint{3.5pt}{3.5pt}}%
{\pgfqpoint{3.3pt}{3.3pt}}%
{%
	\pgfsetlinewidth{0.3pt}%
	\pgfsetdash{{1pt}{1pt}}{0pt}%
	\pgfpathmoveto{\pgfqpoint{0pt}{3.5pt}}%
	\pgfpathlineto{\pgfqpoint{3.6pt}{-0.1pt}}%
	\pgfusepath{stroke}%
}
\makeatother

\newcommand{\assref}[1]{(A\ref{#1})}

\newcommand{\norm}[1]{\left\vert#1\right\vert}
\newcommand{\norms}[1]{\vert#1\vert}
\newcommand{\normR}[1]{\left\lVert#1\right\rVert}
\newcommand{\normRs}[1]{\lVert#1\rVert}

\newcommand{\set}[1]{\left\{ #1\right\}}
\newcommand{\sets}[1]{\{ #1\}}

\newcommand{\EE}[1]{\mathbb{E}\left[#1\right]}
\newcommand{\EEs}[1]{\mathbb{E}[#1]}
\newcommand{\PP}[1]{\mathbb{P}\left(#1\right)}

\newcommand{\1}[1]{\mathbbm{1}_{#1}}

\newcommand{\NN}{\mathbb{N}}
\newcommand{\RR}{\mathbb{R}}

\newcommand{\CX}{\mathcal{X}}
\newcommand{\CY}{\mathcal{Y}}
\newcommand{\CF}{\mathcal{F}}

\newcommand{\CW}{\mathcal{W}}

\newcommand{\CC}{\mathcal{C}}

\NewDocumentCommand{\newzctheorem}{mO{#1}m}{
	\newtheorem{#1}[sharedtheoremcounter]{#3}
	\AddToHook{env/#1/begin}{%
		\zcsetup{countertype={sharedtheoremcounter=#2}}}
}

\newzctheorem{assumption}{Assumption}
\newzctheorem{corollary}{Corollary}
\newzctheorem{definition}{Definition}
\newzctheorem{lemma}{Lemma}
\newzctheorem{proposition}{Proposition}
\newzctheorem{remark}{Remark}
\newzctheorem{theorem}{Theorem}

\zcRefTypeSetup{assumption}{
	Name-sg = Assumption,
	name-sg = assumption,
	Name-pl = Assumptions,
	name-pl = assumptions,
}

\zcRefTypeSetup{corollary}{
	Name-sg = Corollary,
	name-sg = corollary,
	Name-pl = Corollaries,
	name-pl = corollaries,
}

\zcRefTypeSetup{definition}{
	Name-sg = Definition,
	name-sg = definition,
	Name-pl = Definitions,
	name-pl = definitions,
}

\zcRefTypeSetup{lemma}{
	Name-sg = Lemma,
	name-sg = lemma,
	Name-pl = Lemmas,
	name-pl = lemmas,
}

\zcRefTypeSetup{proposition}{
	Name-sg = Proposition,
	name-sg = proposition,
	Name-pl = Propositions,
	name-pl = propositions,
}

\zcRefTypeSetup{remark}{
	Name-sg = Remark,
	name-sg = remark,
	Name-pl = Remarks,
	name-pl = remarks,
}

\newacronym[\glslongpluralkey={Block Markov Decision Processes}]{BMDP}{BMDP}{Block Markov Decision Process}
\newacronym[\glslongpluralkey={Markov Decision Processes}]{MDP}{MDP}{Markov Decision Process}
\newacronym[\glslongpluralkey={Partially Observable Markov Decision Processes}]{POMDP}{POMDP}{Partially Observable Markov Decision Process}
\newacronym[\glslongpluralkey={Rich Observation Markov Decision Processes}]{ROMDP}{ROMDP}{Rich Observation Markov Decision Process}
\newacronym[longplural={Random Sequential Adsorption}]{RSA}{RSA}{Random Sequential Adsorption}
\newacronym[longplural={Spatialized Random Graphs}]{SRG}{SRG}{Spatialized Random Graph}
\newacronym{AA}{AA}{Applied Analysis}
\newacronym{AI}{AI}{Artificial Intelligence}
\newacronym{AMI}{AMI}{Adjusted Mutual Information}
\newacronym{BMC}{BMC}{Block Markov Chain}
\newacronym{BRW}{BRW}{Block Random Walk}
\newacronym{CAIC}{CAIC}{Consistent Akaike Information Criterion}
\newacronym{CASA}{CASA}{Centre for Analysis, Scientific Computing and Applications}
\newacronym{CPSRL}{CPSRL}{Clustered Posterior Sampling for Reinforcement Learning}
\newacronym{CQT}{CQT}{Coherence \& Quantum Technology}
\newacronym{CUDA}{CUDA}{Compute Unified Device Architecture}
\newacronym{CWI}{CWI}{National Research Institute for Mathematics and Computer Science}
\newacronym{DIAM}{DIAM}{Applied Mathematics}
\newacronym{DNA}{DNA}{Deoxyribonucleic acid}
\newacronym{EEMCS}{EEMCS}{Electrical Engineering, Mathematics \&{} Computer Science}
\newacronym{ERRG}{ERRG}{Erd\"{o}s--R\'{e}nyi random graph}
\newacronym{EU}{EU}{European Union}
\newacronym{GPU}{GPU}{Graphics Processing Unit}
\newacronym{GSL}{GSL}{GNU Scientific Library}
\newacronym{HMM}{HMM}{Hidden Markov Model}
\newacronym{INSY}{INSY}{Intelligent Systems}
\newacronym{LMAB}{LMAB}{Latent Multi-Armed Bandit}
\newacronym{LSBM}{LSBM}{Labeled Stochastic Block Model}
\newacronym{MAB}{MAB}{Multi-Armed Bandit}
\newacronym{MCS}{M\&{}CS}{Mathematics and Computer Science}
\newacronym{MC}{MC}{Markov Chain}
\newacronym{MI}{MI}{Mutual Information}
\newacronym{ML}{ML}{Machine Learning}
\newacronym{MSVC}{MSVC}{Microsoft Visual \CC}
\newacronym{NAS}{NAS}{Network Architectures and Services}
\newacronym{NN}{NN}{Neural Network}
\newacronym{ODE}{ODE}{Ordinary Differential Equation}
\newacronym{OFU}{OFU}{Optimism in the Face of Uncertainty}
\newacronym{PSRL}{PSRL}{Posterior Sampling for Reinforcement Learning}
\newacronym{PyPI}{PyPI}{Python Package Index}
\newacronym{QCE}{Q\&{}CE}{Quantum \& Computer Engineering}
\newacronym{QED}{QED}{Quality--and--Efficiency--Driven}
\newacronym{RBM}{RBM}{Reflected Brownian Motion}
\newacronym{RL}{RL}{Reinforcement Learning}
\newacronym{SBM}{SBM}{Stochastic Block Model}
\newacronym{SGD}{SGD}{Stochastic Gradient Descent}
\newacronym{SPECTRA}{SPECTRA}{Sparse Eigenvalue Computation Toolkit as a Redesigned ARPACK}
\newacronym{SVD}{SVD}{Singular Value Decomposition}
\newacronym{TS}{TS}{Thompson Sampling}
\newacronym{TUDelft}{TU Delft}{Delft University of Technology}
\newacronym{TUe}{TU/e}{Eindhoven University of Technology}
\newacronym{UCBVI}{UCBVI}{Upper Confidence Bound Value Iteration}
\newacronym{UCB}{UCB}{Upper Confidence Bound}
\newacronym{UTQ}{UTQ}{University Teaching Qualification}
\newacronym{WiCoS}{WiCoS}{Wireless Communication and Sensing}
\newacronym{EURANDOM}{EURANDOM}{European Institute for Statistics, Probability, Stochastic Operations Research and their Applications}
\newacronym{SPOR}{SPOR}{Statistics, Probability, and Operations Research}
\newacronym{PAC}{PAC}{Probably Approximately Correct}

\newcommand{\doctitle}{Dropout Neural Network Training Viewed from a Percolation Perspective}

\hypersetup
{
    pdftitle={\doctitle},
    pdfauthor={Finley Devlin, Jaron Sanders},
    pdfsubject={\doctitle},
    pdfkeywords={dropout, percolation, stochastic approximation}
}

\begin{document}

\title{\doctitle}

\author{Finley Devlin}
\affiliation{Second Foundation, Czechia}
\author{Jaron Sanders}
\affiliation{Eindhoven University of Technology, The Netherlands}

\begin{abstract}
	In this work, we investigate the existence and effect of percolation in training deep \glspl{NN} with dropout.
	Dropout methods are regularisation techniques for training \glspl{NN}, first introduced by G.\ Hinton et al.\ (2012).
	These methods temporarily remove connections in the \gls{NN}, randomly at each stage of training, and update the remaining subnetwork with \gls{SGD}.
	The process of removing connections from a network at random is similar to percolation, a paradigm model of statistical physics.

	If dropout were to remove enough connections such that there is no path between the input and output of the \gls{NN}, then the \gls{NN} could not make predictions informed by the data.
	We study new percolation models that mimic dropout in \glspl{NN} and characterise the relationship between network topology and this path problem.
	The theory shows the existence of a percolative effect in dropout.
	We also show that this percolative effect can cause a breakdown when training \glspl{NN} without {biases} with dropout; and we argue heuristically that this breakdown extends to \glspl{NN} with biases.
\end{abstract}

\maketitle

\section{Introduction}

\Gls{AI} is increasingly prevalent in many industries.
According to Eurostat, around 30\% of large \Gls{EU} enterprises used \gls{AI} technologies in 2023~\cite{Eurostat:2024:Use}.
These technologies improve worker productivity---potentially bridging a gap between low- and high- skilled workers---and can now perform many tasks better than humans~\cite{Maslej.ea:2024:AI}.

\Gls{ML} underpins many \gls{AI} technologies.
Modern \gls{ML} techniques are often developed in practice as outcome-based approaches, and they often break statistical norms; for example, by relying on overparameterisation~\cite{Allen-Zhu.ea:2019:Convergence, Dar.ea:2021:Farewell}.
This also means that many of the commonly used \gls{ML} methods have outpaced our mathematical understanding.
Developing solid theory behind \gls{ML} is imperative for the challenges faced in \gls{AI} and the future performance of these methods~\cite{Holzinger.ea:2019:Causability}.

We focus in this paper on the training of \glspl{NN}, a \gls{ML} technique.
Training effectively is the largest practical challenge with \glspl{NN}.
This namely requires solving a high dimensional, non-convex optimisation problem and is thus computationally expensive.
For this reason, more computationally efficient methods from stochastic approximation are used for training \glspl{NN}, the most common of which is \emph{\gls{SGD}}~\cite{Kushner.ea:2003:Stochastic}.
Another challenge with \glspl{NN} is to ensure that the model makes high quality predictions on unseen data, and not just the data it was trained on, known as generalisability.
For this, regularisation techniques are employed that modify the training process~\cite{Santos.ea:2022:Avoiding}.

More precisely, we study a regularisation technique known as \emph{dropout}~\cite{Hinton.ea:2012:Improving}.
Dropout methods temporarily omit parts of the \gls{NN} at random, at each step of training.
A stochastic approximation algorithm is then used as normal to update the weights in the remaining subnetwork.
The idea behind this is to force different parts of the network to rely less on one another, enabling them to learn features more independently, and thus improving the generalisability.
Although there is a fair amount of research on dropout's regularisation properties, there is not as much work on the performance of dropout as a stochastic approximation method in its own right~\cite{Senen-Cerda.ea:2025:Almost}.
And while effective, it has, e.g., been observed that training with dropout is slower than with other algorithms~\cite{Garbin.ea:2020:Dropout,Khan.ea:2019:Regularization,Krizhevsky.ea:2012:ImageNet}.
As such, the analytical picture of how and why dropout works well is not yet complete, despite its abundant use in practice.

The random omission of parts of a network relates to \emph{percolation theory}.
Percolation is a paradigm model of statistical physics~\cite{Camia.ea:2019:Probability}, first used to model the flow of fluid through a porous medium in~\cite{Broadbent.ea:1957:Percolation}.
Percolation theory studies whether paths cross the medium, and if so, the fluid is said to percolate.
Seeing the \gls{NN} as the medium, if dropout were to omit parts of the network such that no paths cross it, then information would not percolate.
If information does not flow across the \gls{NN}, then it cannot make predictions informed by data.

This paper contributes to the theoretical understanding of dropout by formalising and investigating this problem.
The insight that we bring is a mathematical proof that percolation can contribute to slower convergence of dropout in deep \glspl{NN} that are not sufficiently wide.
More specifically:

\begin{itemize}
	\item
	      \zcref[S]{chapter:preliminaries} describes new percolation models that mimic dropout in \glspl{NN} and with scalable depth.
	      The scalable depth allows us to model deep learning.

	\item
	      \zcref[S]{chapter:percolation} characterises the relationship between network topology and the aforementioned path problem.
	      This establishes the existence of percolative behaviour in dropout.
	      Notably, \zcref[S]{prop:3:site:2:critical,prop:3:bond:2:critical} establish critical behaviour in the percolation probability; see also \zcref[S]{fig:3:sitephasespace,fig:3:bondparameter-diagram}.

	\item
	      \zcref[S]{chapter:dropout} shows that this path problem can cause a breakdown of dropout when training deep \glspl{NN} without {biases}.
	      \zcref[S]{thm:4:main} ties the performance of the dropout algorithm to the percolation probability.
	      We argue heuristically that \glspl{NN} with biases are also susceptible to such breakdown.
\end{itemize}
The proofs are in the appendices.

Finally, let us emphasize that this paper provides a theoretical characterization of percolation phenomena in \glspl{NN} using dropout.
The results are not enough to practically conclude that percolation explains the aforementioned slowdown observed during the training of large networks.
The present paper, for example, does not establish that practical architectures enter the scaling regime that we study.
Determining whether the mechanism occurs in practice, and whether it measurably affects training dynamics, remains an open empirical question.

\subsection{Previews of the results}

To illustrate the statement ``characterises the relationship between network topology and the aforementioned path problem,'' here is a brief preview of \zcref[S]{prop:3:site:2:trivial, prop:3:site:2:critical}.
They imply that a percolation threshold occurs at a critical scaling of a \gls{NN} using dropout.
To see this, consider a rectangular \gls{NN} of depth $n$ and width $W(n)$ in which dropout removes vertices with probability $p \in (0,1)$.
\zcref[S]{prop:3:site:2:trivial} shows that if $W(n) = o( \log{n} )$ or $W(n) = \omega( \log{n} )$, its percolation function, i.e., the crossing probability, converges to $0$ or $1$ as $n \to \infty$, respectively.
\zcref[S]{prop:3:site:2:critical} states, for example, that if $W(n) = c \log{n} + o(1)$ with $c > 0$, then there exists a critical dropout probability $p_c \in (0,1)$ such that if $p > p_c$, $p = p_c$, or $p < p_c$, the percolation function converges to $0$, a constant in $(0,1)$, or $1$, respectively.
In other words, there is a non-degenerate percolation threshold at $p_c$ when the width of a \gls{NN} scales proportionally to $\log{n}$, and not in \glspl{NN} scaled more gently or intensively in width.

To explain the statement ``the path problem can cause dropout to break down in deep \glspl{NN},'' here is a preview \zcref[S]{cor:4:constantW}.
This corollary of \zcref[S]{thm:4:main} implies, for example, that if one trains
\begin{enumerate}[noitemsep, label=(\roman*)]
	\item a rectangular neural network of depth $n$, width $W$, and without biases,
	\item using dropconnect for $T(n)$ iterations, with step sizes $\alpha_t > 0$, so each weight is dropped independently with probability $p$ in each iteration,
\end{enumerate}
then the weights can move at most on the order of
$
	\exp{( - n p^{W^2} )}
	\sum_{t=0}^{T(n)-1}
	\alpha_t
$
from their initialisation.
For example, if $p = 1/2$ and $\alpha_t = \alpha / (t+1)$, this implies that in very deep \glspl{NN}, $T(n)$ must be at least super-exponential in $n$ for the weights to move a constant distance.
In conclusion, compensating for the lack of percolation contributes to the many iterations required to train a deep \gls{NN} with dropout.

\section{Preliminaries}\label{chapter:preliminaries}

We now cover the relevant preliminary knowledge and notation across \glspl{NN}, dropout, and percolation theory.
Throughout, we work with probability space $(\Omega, \CF, \mathbb{P})$, and all random variables are assumed measurable.

\subsection{Deep learning}\label{sec:2:deep_learning}

We consider feedforward \glspl{NN} which do not have loops.
Each layer in such \gls{NN} consists of activation functions and parameters called weights and biases.
The activation functions $\sigma:\RR \to \RR$ are defined as real functions, but we often apply them pointwise over vectors; e.g., $\sigma(x) := (\sigma(x_1), \dots, \sigma(x_n))$ for $x\in\RR^n$.
Let $\CX$ be the \emph{feature space} and $\CY$ the \emph{label space} of the \glspl{NN} with dimensions $d_1, d_2 \in \NN$, respectively.
Typically, $\CX = \RR^{d_1}$ and $\CY = \RR^{d_2}$.

\begin{definition}[Deep feedforward \gls{NN}]\label{def:2:deep_feedforward_nn}
	Let $L\in\NN$, $W_0, \dots, W_{L+1} \in\NN$, and $\sigma_1, \dots , \sigma_{L+1}$ be activation functions corresponding to each (non-input) layer.
	Define affine transformations $S_\ell :\RR^{W_{\ell - 1}} \to \RR^{W_\ell}$ by $S_\ell (x) := A_\ell x + b_\ell$ where $A_\ell \in \RR^{W_\ell \times W_{\ell -1}}, \; b_\ell \in \RR^{W_\ell}$ for $\ell = 1, \dots, L+1$.
	A \emph{feedforward \gls{NN}} is the model $F: \CX \times \CW \to \CY$ defined by
	\begin{align}
		F(x, w)
		:=
		\sigma_{L+1} \circ S_{L+1}\circ \cdots \circ \sigma_1 \circ S_1(x),
	\end{align}
	for $(x,w) \in \CX \times \CW$.
	Here, the parameters $w = (w_1, \ldots, w_{L+1}) \in \CW:=\CW_1 \times \cdots \times \CW_{L+1} $ are given by $w_\ell = (A_\ell, b_\ell) \in\RR^{W_\ell \times W_{\ell -1}} \times \RR^{W_\ell}$.
\end{definition}

In \zcref[S]{def:2:deep_feedforward_nn}, $L$ denotes the number of hidden layers, yielding $L+2$ total layers including input and output.
The network’s \emph{depth} refers to its total number of layers; it is \emph{deep} if $L>1$ and \emph{shallow} if $L=1$.
The quantities $W_0, \dots, W_{L+1}$ specify the number of neurons per layer (the \emph{widths}).
The parameters $A_1, \dots, A_{L+1}$ are the \emph{weights}, and $b_1, \dots, b_L$ the \emph{biases}.
The $i$th neuron in layer $\ell$ is indexed by $(i,\ell)$.

\zcref[S]{fig:2:nn_schematic} depicts a deep feedforward \gls{NN}.
A neuron can be thought of as a particle to which an activation function is applied.
Each neuron is given a real value input as a linear combination of the outputs of the previous layer with some added bias.
The biases can be seen as extra neurons which always output the value 1.
Each edge between neurons in \zcref[S]{fig:2:nn_schematic} denotes a weight and thus the edges between the layers collectively depict the affine transformations, while the neurons in the non-input layers represent the activation functions.

\begin{figure}[htbp]
	\centering
	\includegraphics[width=0.618\linewidth]{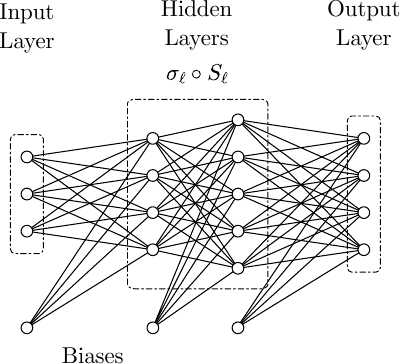}
	\caption{A schematic representation of a fully connected deep \gls{NN} with $L=2$ hidden layers.}%
	\label{fig:2:nn_schematic}
\end{figure}

The connections between neurons form an important part of the structure of the \gls{NN}.
In \zcref[S]{fig:2:nn_schematic} all the layers are fully connected, but if a weight is zero between two neurons, then they are functionally no longer connected.
We refer to the graph consisting of the neurons and their connections with non-zero weights (excluding the biases) as the \emph{neural network connectivity graph}.
An example is depicted in \zcref[S]{fig:2:nncongraph}.

\begin{definition}[Neural network connectivity graph]\label{def:2:NN_connec}
	Let $F( \cdot, w)$ be a deep feedforward \gls{NN} with parameters $w$.
	Define a vertex set $V$ with one vertex for each neuron.
	Define a directed edge set $E$ such that $((i, \ell), (j, \ell+1)) \in E$ if ${[A_\ell]}_{ij} \neq 0$.
	Let $G(F( \cdot, w)):= (V, E)$ be known as the \emph{connectivity graph} of $F( \cdot, w)$.
\end{definition}

\begin{figure}[htbp]
	\centering
	\includegraphics[width=\linewidth]{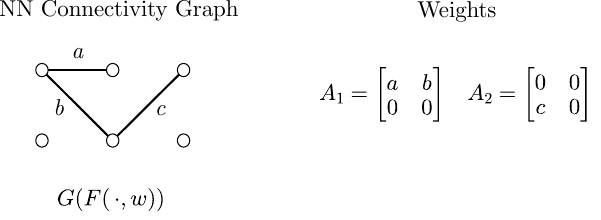}
	\caption{%
		An example of a \gls{NN} connectivity graph for a \gls{NN} with $L=1$ hidden layers, and corresponding weights $A_1, A_2$.
		Note that here, we presume that $a,b,c\neq 0$.
	}%
	\label{fig:2:nncongraph}
\end{figure}

In \gls{ML} terms, the architecture of the \gls{NN} refers to the choices $L$, $W_0$, $\dots$, $W_{L+1}$, and $\sigma_1$, $\dots$, $\sigma_{L+1}$.
This architecture forms a family of \glspl{NN} parametrised by $w$, which forms the set of admissible functions used in the supervised learning problem.
Similar such models are studied in, e.g.,~\cite{Jacot.ea:2018:Neural, Geiger.ea:2020:Disentangling}.

For simplicity, we assume that the layers have constant width, i.e., $W_\ell = W$ for $\ell = 0, \dots, L+1$, and are fully connected.
Little generality is lost because weights or biases can be set to zero.
A smaller network can also always be embedded into a larger, constant width network.

Historically, focus was placed on \emph{shallow \glspl{NN}} (when $L=1$) over deep \glspl{NN} (when $L>1$) partially because shallow \glspl{NN} already satisfy \emph{universal approximation} for non-polynomial activations~\cite{Leshno.ea:1993:Multilayer}.
However, in deep \glspl{NN}, activation functions are composed over the layers.
This allows one neuron to capture more complex curvature.
As such, deep \glspl{NN} can approximate a function at a faster rate than shallow \glspl{NN} using the same number of parameters~\cite{DeVore.ea:2021:Neural, Elbrachter.ea:2021:Deep}.
This motivates the move to deep \glspl{NN} from an approximation theory perspective.

Motivated by approximation theory and performance improvements, \glspl{NN} are becoming increasingly deep.
In this work, we consider the potential issues with training increasingly deep \glspl{NN} using dropout.
These issues are connected with percolation theory and for this, we consider the limiting case as $L \to \infty$.
This forms an important step in understanding the finite depth case.

\subsection{Dropout}\label{sec:2:dropout}

Dropout is a class of methods used for training \glspl{NN} in combination with \gls{SGD}.

\begin{definition}[\gls{SGD}]\label{def:2:SGD}

	Let $w\mapsto R(w)$ be a differentiable \textit{objective function} over parameter space $\CW$.
	The stochastic recursion
	\begin{equation}
		w_{t+1} = w_t - \alpha_t g(w_t, \xi_t)
	\end{equation}
	is called \emph{stochastic gradient descent} if
	\begin{equation}\label{eq:2:SO:unbiased}
		\EE{g(\,\cdot\, , \xi_t)} = \nabla R(\,\cdot\,)
		\quad
		\textnormal{for all}
		\quad
		t\in\NN
		.
	\end{equation}
\end{definition}

\gls{SGD} can be considered a stochastic approximation of the gradient descent algorithm.
In studies of stochastic approximations, see, e.g.,~\cite{Kushner.ea:2003:Stochastic}, the following assumptions are typical to obtain convergence guarantees:

\begin{assumption}\label{ass:2:SO:second_moment}
	$\sup_{t\in\NN} \EEs{\normR{g(\,\cdot\,, \xi_t)}^2} < \infty$
\end{assumption}

\begin{assumption}\label{ass:2:SO:learning_rates}
	$\sum_{t=0}^\infty \alpha_t = \infty$, $\alpha_t \geq 0$ for all $t\in\NN$ and $\alpha_t\to0$ as $t\to\infty$
\end{assumption}

We need \emph{not} adopt \zcref[S]{ass:2:SO:second_moment,ass:2:SO:learning_rates} for our results.
Rather, \zcref[S]{cor:4:constantW} relies only on the first moment being finite (see \zcref[S]{ass:4:second_moment}), and emphasizes the necessity of \zcref[S]{ass:2:SO:learning_rates}.

The idea of dropout is to modify \gls{SGD} such that at each training step, nodes (or weights) are removed randomly.
The proposed purpose of this is to reduce the correlation between different weights, meaning neurons can learn different features and improve generalisation of the network~\cite{Hinton.ea:2012:Improving, Srivastava.ea:2014:Dropout, Wan.ea:2013:Regularization}.
This is more broadly known as regularisation.

Dropout was first proposed in~\cite{Hinton.ea:2012:Improving}, in which each neuron and all their associated weights are filtered independently with probability $1/2$.
Following this, many dropout algorithms have been studied.
Notably,~\cite{Wan.ea:2013:Regularization} introduced \emph{dropconnect} which filters each weight independently with probability $p$.
In~\cite{Wan.ea:2013:Regularization}, dropconnect showed improved performance and reduced overfitting of digit recognition of the MNIST database~\cite{Deng:2012:MNIST}.

In this work, we refer to any dropout algorithm simply as dropout, and focus especially on dropconnect and the original dropout for their relations to bond and site percolation.
To that end, we now define a family of stochastic recursions that we will refer to as \emph{dropout \gls{SGD}}.

\begin{definition}[Dropout \gls{SGD}]\label{def:2:dropout_sgd}

	Let $w \mapsto R(w)$ be a differentiable \textit{objective function} over parameter space $\CW$ and ${(f_t)}_{t\in\NN}$ be a sequence of $\set{0,1}$-valued random vectors in parameter space $\mathcal{W}$ known as the \emph{filter vectors}.
	The stochastic recursion
	\begin{align}
		w_{t+1}
		=
		w_t - \alpha_t f_t\odot g(f_t\odot w_t, \xi_t)
	\end{align}
	is called \emph{dropout \gls{SGD}} if \zcref[S]{eq:2:SO:unbiased} holds.
	Here, $\odot$ denotes the component-wise product.
\end{definition}

\zcref[S]{def:2:dropout_sgd} describes a modification of \gls{SGD} whereby the gradient estimate is taken over the filtered network $f_t\odot w_t$, and only the non-filtered weights are updated.
The distribution of the sequence of filter vectors determines the specific dropout algorithm.

Dropconnect is described by filters $f$ such that each element corresponding to a weight is independently 0 with probability $p$ and $1$ with probability $1-p$.
In original dropout the filters are not independent: each neuron is filtered with probability $p$, and if a neuron is filtered, then all the connected \emph{weights} are filtered in $f$.
Note that biases are not typically filtered.
Two examples of filtered networks for dropconnect and dropout can be found in \zcref[S]{fig:2:dropout_samples}.

\begin{figure*}
	\centering
	\subfloat[Dropconnect with $p=0.5$]{\includegraphics[width=0.35\linewidth]{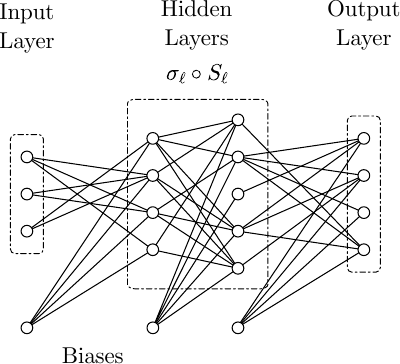}}
	\quad
	\subfloat[Original dropout with $p=0.5$]{\includegraphics[width=0.35\linewidth]{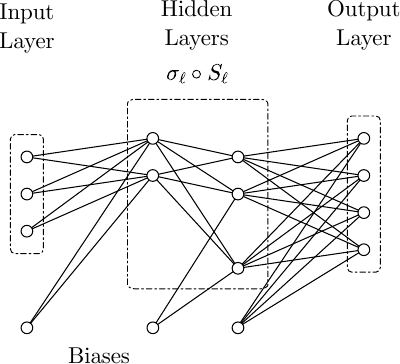}}
	\caption{Two schematics of \glspl{NN} $F( \cdot , f \odot w)$, with two different distributions of dropout filters.
		Observe that edges are dropped from each layer for dropconnect, but only in the hidden layers for original dropout.
		The unfiltered \gls{NN} $F( \cdot , w)$ is depicted in \zcref[S]{fig:2:nn_schematic}.
	}%
	\label{fig:2:dropout_samples}
\end{figure*}

Note that by modifying the stochastic approximation algorithm, the gradient estimates no longer correspond to the \gls{SGD} objective $R(\cdot)$ but rather some different implied \emph{dropout} objective.
This dropout objective is characterised in~\cite{Senen-Cerda.ea:2025:Almost}, along with convergence guarantees, and is defined as the solution of an ODE arising from the expected gradient estimates.
In \zcref[S]{chapter:dropout}, we discuss the dropout objective in the case of the ideal gradient estimate when the objective $R(w) := \EE{\ell(F(X, w), Y)}$ is defined in terms of a \emph{loss function} $(x,y)\mapsto \ell(x,y)$, which is most typical.

Our work revolves around the following potential issue with dropout.
Consider the scenario that the sampled filters $f$ are such that no path connects the input and output layers in the \gls{NN} connectivity graph $G(F( \cdot , f \odot w))$.
Consequently, $F( \cdot, f \odot w)$ is a constant function (\zcref[S]{lemma:4:no_path}), and any gradient estimate is therefore independent of the input data.
Under this scenario, a key component of supervised learning is lost: we have labels but with no corresponding input.
Thus, philosophically, no learning is contained in the dropout gradient estimate in this scenario.

\subsection{Percolation theory}

The canonical model in percolation theory considers the square lattice $\mathbb{Z}^2$ in which each edge of the infinite graph is removed independently with probability $p$.
This model is known as \emph{bond} percolation, and an example of this on a finite lattice can be found in \zcref[S]{fig:2:lattice}.
The classical question on this model is whether there exists an infinitely large connected component.

\begin{figure}[htbp]
	\centering
	\subfloat[$p=0.53$]{\includegraphics[width=0.4\linewidth]{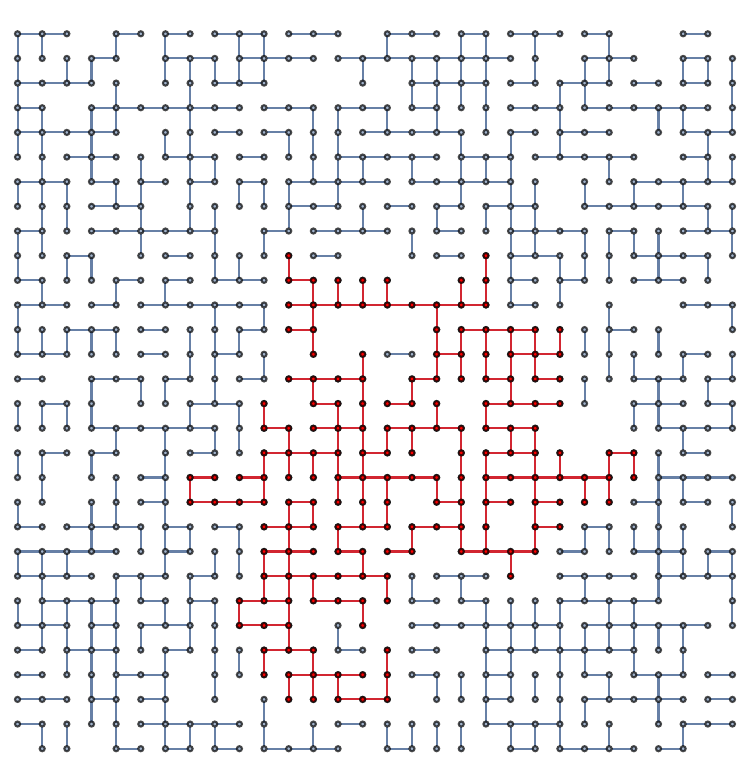}}
	\qquad
	\subfloat[$p=0.47$]{\includegraphics[width=0.4\linewidth]{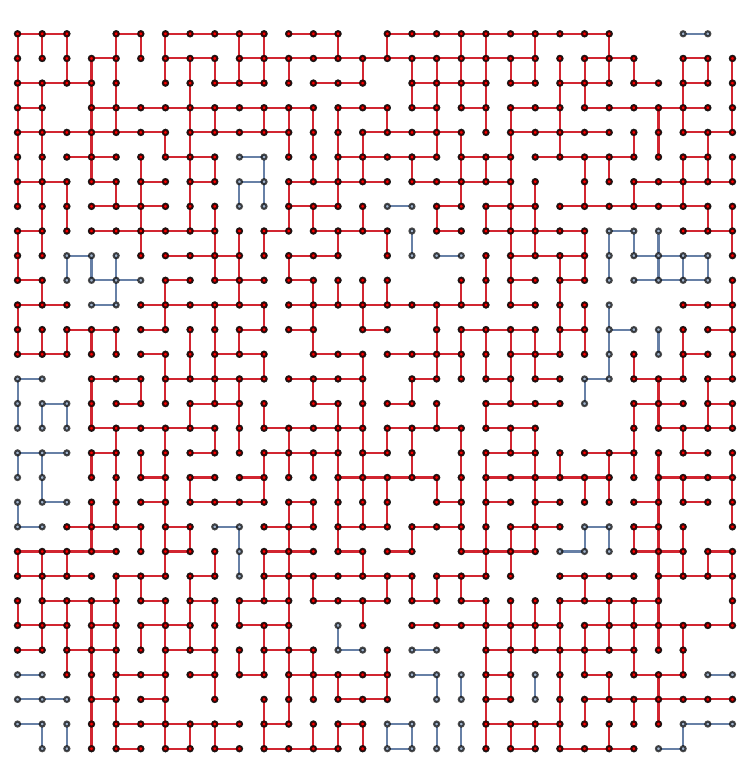}}
	\caption{Two samples of bond percolation on a $30\times 30$ square lattice, each with the largest connected component in red.}
	\label{fig:2:lattice}
\end{figure}

We can define another percolation model known as \emph{site} percolation by instead independently removing each vertex, and their edges, with probability $p$.
In this work we consider both bond and site percolation for their respective connections to dropconnect and original dropout.
In \zcref[S]{sec:2:perc:model} we define our dropout inspired percolation models and in \zcref[S]{sec:2:perc:crossing} we define our object of interest, the \emph{percolation function}, which describes the phase transition we wish to study.

\begin{remark}
	In percolation literature, edges are typically removed with probability $1-p$, but in this work, we use $p$ for readability of the proofs and results.
\end{remark}

\subsubsection{Model}\label{sec:2:perc:model}

We now define percolation models that describe the connectivity graphs of \glspl{NN} when dropout filters are applied.
These graphs consist of $L$ hidden layers of $W$ neurons, where each node connects to every node in the next layer, as outlined in \zcref[S]{sec:2:deep_learning}.
Additionally, our \glspl{NN} are feedforward and thus the edges between neurons are directed.

\begin{definition}[Rectangular Layered Network]
	Let $L, W \in \NN$ and define the set of vertices $V:= \set{1, \dots,W} \times \set{0, \dots, L+1}$.
	Let $E$ be the edge set containing the directed edges $(v_{(i, \ell)}, v_{(j, \ell+1)}) \in E$ for all $i,j=1,.
		..\dots W$ and $\ell=0, \dots, L$.
	Then $G = (V,E)$ is known as the $W\times L$\emph{ Rectangular Layered Network}.
	Additionally, $L$ is known as the \emph{length} and $W$ the \emph{width} of the graph, where there are $L+2$ layers each of $W$ vertices.
	Vertex $v_{(i, \ell)} \in V$ denotes the $i$th node in the $\ell$th layer.
	The 0th and $L+1$st layers are referred to as the \emph{input} and \emph{output} layers respectively.
\end{definition}

\begin{remark}
	Extensions to general layer widths are possible, but would be cumbersome while adding limited additional insight.
\end{remark}

By construction, the rectangular layered network is the same as the connectivity graph of our \gls{NN} model (fully connected, constant width, and feedforward) when all the weights are non-zero.
When we train the network with dropconnect, at each time step some weights are filtered by setting the weight to zero.
A zero weight is equivalent to there being no connection between the two corresponding vertices.
We can therefore consider this edge temporarily \emph{deleted}.
This is the precise link between the process of dropout and percolation, leading to the following definition:

\begin{definition}[Rectangular Layered Percolation]\label{def:rect_layered_perc}
	Let $L, W \in \NN$, $p\in [0,1]$, and $H = (V,E)$ be the $W\times L$ Rectangular Layered Network.
	Consider the following two percolation models:%
	\begin{enumerate}[label=(\roman*)]
		\item
		      \emph{Bond percolation.}
		      Remove each edge $e\in E$ with probability $p$, independently of one another.
		      Let $G=(V, E')$ be the resulting graph, with distribution $G^{\mathrm{bond}}(p, W, L)$, say.
		\item
		      \emph{Site percolation.}
		      Remove each vertex in $U=\set{1, \dots,W} \times \set{1, \dots, L}$ independently with probability $p$, independently of one another.
		      Let $G$ be the induced graph of $H$ over the remaining vertices, with distribution $G^{\mathrm{site}}(p, W, L)$, say.
	\end{enumerate}
\end{definition}

The difference between the bond and site percolation models is depicted in \zcref[S]{fig:2:perc_samples}.
The former corresponds to the dropconnect algorithm, and the latter to the original dropout algorithm.
Note that in the original dropout algorithm, only vertices in hidden layers are filtered.
This is reflected in \zcref[S]{def:rect_layered_perc} by omitting layers $0$ and $L+1$ in the subset $U$.

\begin{figure*}
	\centering
	\includegraphics[width=0.9\linewidth]{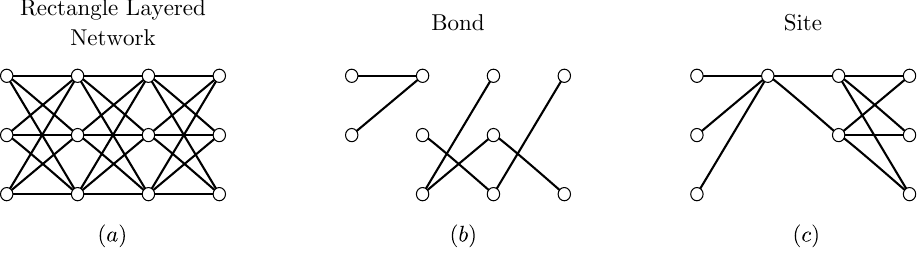}
	\caption{%
		An example of a rectangular layered network with $W=3, L=2$ in (a) and samples of bond and site percolation on this rectangular layered network in (b) and (c) respectively.
		Observe that vertices are not removed from the first or last layers in site percolation.
	}
	\label{fig:2:perc_samples}
\end{figure*}

Let $F$ be a fully connected \gls{NN} with $L$ hidden layers and constant width $W$, and let $f$ be corresponding filters with a dropconnect or original dropout distribution with parameter $p$.
Then by construction, the connectivity graph of the filtered network $G(F( \cdot , f))$, has distribution $G^{\text{bond}}(p, W, L )$ or $G^{\text{site}}(p, W, L )$, respectively.

\subsubsection{Crossing probability}\label{sec:2:perc:crossing}

The probability that a path connects the input and output layers is known as the \emph{crossing probability}.
In the canonical model, this crossing probability is characterised in the Russo--Seymour--Welsh theorem, which considers the probability that there exists a path from left to right of some $n\times kn$ box belonging to the lattice.
In our model, the crossing probability depends on the percolation parameter $p$ and network size $W\times L$.
We define this relationship as the \emph{percolation function}:

\begin{definition}[Percolation function]\label{def:2:percolation_function}
	Let $W, L \in \NN, p\in[0,1]$ and $G\sim G(p, W, L)$.
	Let $\mathcal{C}(G)$ be the set of vertices in the output layer that are connected to the input layer, thus
	\begin{equation}
		\mathcal{C}(G)
		:=
		\set{ j\in [W] : \exists i \in [W] \textnormal{ s.t.\ } (i,0) \leftrightarrow (j, L+1) }.
	\end{equation}
	We define the map $(p, W, L)\mapsto \theta(p, W, L)$ known as the \emph{percolation function} by
	\begin{equation}
		\theta(p, W, L)
		:=
		\PP{\norm{\mathcal{C}(G)} > 0}
		.
	\end{equation}
\end{definition}
Note that the percolation function depends on the percolation model, which, when the distinction is necessary, will be denoted by $\theta^{\text{bond}}$ or $ \theta^{\text{site}}$.
\zcref[S]{fig:2:perc_func_noscale} depicts the shape of the site percolation function as each parameter changes, with the other two parameters fixed.

These shapes are informed by our characterisation of the site percolation function in \zcref[S]{prop:3:site:1}.
As $L$ grows large, each potential path becomes longer and is thus less likely exist.
On the other hand, as $W$ grows, so too do the number of distinct potential paths, and therefore the probability that one of them reaches the output layer tends to 1.
We also see that the percolation function undergoes a rapid change around some value of $p$.
This alludes to the existence of a so-called \emph{phase transition} of the percolation function in $p$.

\begin{figure}[htbp]
	\centering
	\includegraphics[width=0.9\linewidth]{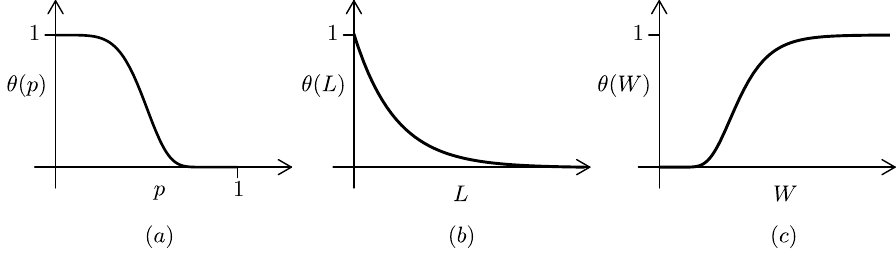}
	\caption{%
		Three figures depicting the shape of the site percolation function $\theta(p, W, L)$ when each of $p,W, L$ is varied.
		The shapes are given by \zcref[S]{prop:3:site:1}.
	}
	\label{fig:2:perc_func_noscale}
\end{figure}

Phase transitions are routinely observed in percolation models on infinite graphs.
To investigate phase transitions, we must consider limits of the percolation function as network grows in size.
We do so by letting the width and length $W, L:\NN \to\NN$ be maps of some index to network size, and then consider the limits of these networks.
The concept of a phase transition is formalised below with the definition of the critical threshold.

\begin{definition}[Critical percolation threshold]
	Let $W=W(n)$, $L=L(n)$ be such that $W(n),L(n)\to\infty$ as $n\to\infty$.
	Let $G_n \sim G(p, W(n), L(n))$ for all $n\in\NN$ and define $\theta_n(p) := \theta(p, W(n), L(n)) = \PP{\norm{\mathcal{C}(G_n)} >0}$.
	The \emph{critical percolation threshold} $p_c$ is then given by
	\begin{equation}
		p_c
		:=
		\inf\set{p\in [0,1] \;:\; \lim_{n\to\infty}\theta_n(p) = 0 }
		.
	\end{equation}
\end{definition}

Note that $\theta_n(p)\in[0,1], \theta_n(1) = 0, $ and $ \theta_n(0) = 1$, thus $p_c$ is well-defined.
The critical threshold defines a concept of phase transition because for $p\geq p_c$ there is no path with high probability, but below the threshold there is a positive probability that a path exists and thus a small change in the parameter yields a large change in the behaviour of the system.
It is important to observe that the critical threshold $p_c = p_c(L, W)$ depends on the specific scalings of $L, W$.
In \zcref[S]{chapter:percolation} we let $L(n) = n$ and analyse the dependence of $p_c$ on the scaling of $W$.

\begin{remark}
	In the above scaling regime we choose to fix $p$ rather than scale it.
	The results can in principle be extended to scenarios in which $p$ scales with $n$.
\end{remark}

\section{Related literature}

\subsection{Percolation theory and random graphs}

To our knowledge, the specific percolation model studied in this paper had not previously appeared in literature.
Although several of its individual features had been considered separately, their combination had not.
Related work studies the large-width asymptotics of dropout in two-layer neural networks~\cite{Chizat.ea:2025:Phase}.
Like the present paper, it identifies distinct limiting regimes arising from the joint scaling of architectural and algorithmic parameters, but in a different setting: we study the interplay among network depth, width, and training time, whereas~\cite{Chizat.ea:2025:Phase} studies the interplay among network width, learning rate, and dropout rate.

Most attention in percolation has been on the size of the giant component, although the crossing probability has also seen attention, see, e.g.,~\cite{Watts:1996:crossing, Cardy:1992:Critical, Langlands.ea:1992:universality, Kohler-Schindler.ea:2023:Crossing}.
Much of this work considers \emph{planar} percolation models which operate upon an infinite lattice.
In our problem, instead, the degree of each vertex grows as the graph grows.

Graph models with growing degrees are common in random graph theory.
Our model allows vertex degrees that diverge in the limit, as in~\cite{Esker.ea:2005:Distances}, which studies distances in the configuration model with a power-law degree distribution.
In contrast, our model assigns degrees that are finite with probability tending to $0$ as the graph grows.
Its binomial degree distribution resembles that of the \emph{Erd\H{o}s--Rényi} random graph~\cite{Erdos.ea:1959:Random}, although that model is typically studied in a sparse regime.
Recent work~\cite{Lichev.ea:2024:Percolation} characterises the size of the giant component in bond percolation on dense configuration models.
While this shares degree and scaling features with our model, it differs because of the layered geometry \glspl{NN} have.

One problem involving growing dimension and geometric structure is bond percolation on the Hamming hypercube $\set{0,1}^m$, especially identifying the critical probability at which a giant connected component exists~\cite{Hofstad.ea:2017:Hypercube}.
However, the cube's geometry differs from our layered structure.
Percolation on layered graphs is studied in~\cite{Guha.ea:2016:Spanning}, where independent identically distributed percolation samples are stacked with edges between layers.
Dropout algorithms instead removes edges between layers.
While these differences matter for bond percolation, there is some similarity to site percolation, since edges between layers exist only when corresponding vertices are present.
Site percolation in two-layer networks is examined in~\cite{Cao.ea:2021:Percolation}, although our model scales both the number of layers and vertices per layer.

Finally, our graph is directed.
This reduces symmetry and adds complications motivating its own study~\cite{Henkel.ea:2008:Non-Equilibrium}.

\subsection{Convergence of dropout}

Literature on the convergence of dropout is the most relevant to our work.
This literature is sparse but has seen recent attention.

A first stochastic approximation perspective on dropout appeared on arXiv in 2020~\cite{Senen-Cerda.ea:2025:Almost}.
Proposition~\cite[Proposition~6]{Senen-Cerda.ea:2025:Almost} proves that the iterates of dropout \gls{SGD}, for \glspl{NN} without biases, converge to the limiting set of an \gls{ODE}.
Under sufficient regularity of the activation functions and dropout algorithm, this shows that the iterates converge to the set of critical points almost surely.
This result used the classical tool known as the \emph{\gls{ODE} method}~\cite[Chapter~5]{Kushner.ea:2003:Stochastic}.
The result largely avoided issues of percolation and concluded that dropout algorithms are well-behaved asymptotically.
Proposition~\cite[Proposition~7]{Senen-Cerda.ea:2025:Almost} characterises convergence in a weak sense and shows that the rate of convergence to stationary points increases if the edge removal probability decreases.
Finally,~\cite[Proposition~9]{Senen-Cerda.ea:2025:Almost} characterises the rate of convergence of dropout to the minimiser when the underlying network is an arborescence.
The result finds exponential decay with a rate smaller than $O({(1-p)}^n)$, which shows a clear dependence on the dropout probability.

The difference with our work lies in the perspective on where issues may lie.
We focus on the potential breakdown of dropout due to possibly no path crossing the network, whereas~\cite{Senen-Cerda.ea:2025:Almost} argues that dropout is well-behaved asymptotically and investigates whether the convergence is slowed by dropout.
For our perspective, we identify the precise scaling of the topology and training time needed to avoid such a breakdown.
Dropout is \emph{allowed} to misbehave asymptotically (\zcref[S]{thm:4:main}), and its objective \emph{may} be of poor quality (\zcref[S]{sec:4:int:algo}).

Between 2020 and 2022,~\cite{Senen-Cerda.ea:2022:Asymptotic, Mianjy.ea:2020:Convergence} both studied the convergence of dropout in shallow \glspl{NN}, although with different techniques and assumptions.
In~\cite{Senen-Cerda.ea:2022:Asymptotic}, the convergence rate of linear shallow \glspl{NN} is shown, given that the iterates follow the trajectories of the gradient flow identified by the \gls{ODE} method.
In particular, it is shown that under this gradient flow, the trajectories converge exponentially fast, with an exponent depending on the dropout parameter, when close to a minimiser.
Reference~\cite{Mianjy.ea:2020:Convergence} studies the convergence of dropout in shallow ReLU \glspl{NN}.
They show, under assumptions of data separability and lazy training, that $\varepsilon$-optimality of the objective can be reached with $O(1/\varepsilon)$ training steps.
Lazy training refers to assuming that the iterates of \gls{SGD} do not move far from their initialisation.
Different to~\cite{Senen-Cerda.ea:2022:Asymptotic}, the convergence rate in~\cite{Mianjy.ea:2020:Convergence} is independent of the percolation probability, which they argue is due to separability.
We make no such assumptions, and in particular, we consider deep \glspl{NN} to connect to percolation.

Also in 2022,~\cite{Mohtashami.ea:2022:Masked} obtained convergence results for a generalised version of \gls{SGD} which allows for perturbations and filters.
In particular,~\cite{Mohtashami.ea:2022:Masked} shows that $\varepsilon$-optimality of the averages of $\normRs{\nabla R(w_t)}^2$ can be reached in $O(1/\varepsilon^2)$ training steps.
Choosing dropout filters recovers a result similar to~\cite[Proposition~7]{Senen-Cerda.ea:2025:Almost}.
Characterising the average of means has precedent as a convergence result for non-convex objectives, for example~\cite[Theorem 4.8]{Bottou.ea:2018:Optimization}, which can be recovered by~\cite[Theorem~1]{Mohtashami.ea:2022:Masked}.
Although this is a weak characterisation of convergence, it enables a general result with mild assumptions.
In \zcref[S]{chapter:dropout} we characterise convergence more directly by looking at the evolution of the parameters themselves.

Most recently in 2024,~\cite{Shalova.ea:2024:Singular-limit} shows convergence of a general class of \emph{noisy} gradient descent algorithms.
This class includes dropconnect and the original dropout.
In particular,~\cite{Shalova.ea:2024:Singular-limit} shows that when the learning rate and dropout parameter are small, the time scaled iterates follow the trajectories of an \gls{ODE} towards the manifold of minimisers, when starting in its basin of attraction.
Moreover, they characterise the behaviour of the scaled iterates when they reach the MM and continue to evolve along it.
While the results of~\cite{Shalova.ea:2024:Singular-limit} hold for a general class of noisy gradient descent algorithms, they scale algorithm hyperparameters and not topology like we do in \zcref[S]{chapter:dropout}.

\subsection{Regularisation of dropout}

Most research on dropout has focused on its regularisation properties, possibly because this is the intended purpose of dropout from a practical perspective.
The regularisation property was first observed in~\cite{Hinton.ea:2012:Improving, Srivastava.ea:2014:Dropout} and again for dropconnect in~\cite{Wan.ea:2013:Regularization}.
The property has since been well studied, see, e,g,~\cite{Baldi.ea:2013:Understanding, Baldi.ea:2014:dropout, Wager.ea:2013:Dropout, Mianjy.ea:2018:Implicit, Wei.ea:2020:implicit}.

This work does not study the regularisation property.

\section{Percolation Theory for Dropout \texorpdfstring{\glspl{NN}}{Neural Networks}}\label{chapter:percolation}

This section characterises the percolation function for our bond and site percolation models.
In particular, we find critical scalings of the topology such that the critical percolation threshold is non-trivial.

\zcref[S]{sec:3:site} contains results for the dropout-site percolation model.
\zcref[S]{sec:3:bond} considers the more challenging dropconnect--bond percolation model.
In \zcref[S]{sec:3:connection}, we show that the upper bound on bond percolation is recoverable from the results on site percolation.

\subsection{Site percolation}\label{sec:3:site}

We obtain a complete characterisation of the percolation function in \zcref[S]{prop:3:site:1} by exploiting its equivalence to the probability there are no cuts between layers.

\begin{proposition}\label{prop:3:site:1}
	$\theta^{\text{site}}(p, W, L) = {(1-p^W)}^{L}$
\end{proposition}

The following results characterise the limiting behaviour of the percolation function, in which it exhibits critical behaviour in both topology and percolation probability, as we scale $W, L$.

\begin{proposition}[Noncritical scalings]\label{prop:3:site:2:trivial}
	Let $p\in(0,1), L(n) = n$.
	The following then holds:
	\begin{enumerate}[label=\roman*.]
		\item If $W(n) = \omega(\log n)$, then $\theta_n^{\text{site}}(p) \to 1$ as $n\to\infty$ and $p_c = 1$.
		\item If $W(n) = o(\log n)$, then $\theta_n^{\text{site}}(p) \to 0$ as $n\to\infty$ and $p_c = 0$.
	\end{enumerate}
\end{proposition}

\zcref[S]{prop:3:site:2:trivial} demonstrates a phase transition in the topology around logarithmic scaling.
This is depicted in \zcref[S]{fig:3:sitephasespace}(a).
Then, under logarithmic scaling, we find a non-degenerate critical probability:

\begin{proposition}[Critical scaling]\label{prop:3:site:2:critical}

	Let $p\in(0,1), c>0$, $L(n) = n$.
	If $W(n)/\log n \to c $ as $n\to\infty$, then $p_c = \exp(-1/c)$.
	Furthermore:
	\begin{enumerate}[label=\roman*.]
		\item If $p < p_c$ then $\theta_n^{\text{site}}(p) \to 1$ as $n\to\infty$.
		\item If $p > p_c$ then $\theta_n^{\text{site}}(p) \to 0$ as $n\to\infty$.
		\item If $p = p_c$ and $W(n) = c \log{n} + o(1)$, then $\theta_n^{\text{site}}(p) \to e^{-1}$ as $n\to\infty$.
	\end{enumerate}
\end{proposition}

\zcref[S]{prop:3:site:2:critical} gives a complete view of the phase transition in the percolation probability under the critical logarithmic scaling.
This is depicted in \zcref[S]{fig:3:sitephasespace}(b).
The rate $c$ smoothly controls the critical probability, whereby any $p_c\in(0,1)$ can be achieved with $W(n) = -\log(n)/\log(p_c) + o(1)$, explicitly connecting the interaction of the topology and percolation parameter.
This describes a critical plain of $(W(n), p)$, along which \zcref[S]{prop:3:site:2:critical}.iii states that the percolation function is constant.
Additionally, there exists a fixed-point like relationship at $c=1$ whereby $ \theta^{\text{site}}_n(p_c) \to p_c$ as $n\to\infty$.

\begin{figure}[htbp]
	\centering
	\resizebox{\linewidth}{!}{\begin{tikzpicture}[
	>={Stealth[length=2.5mm]},
	line cap=round,
	font=\normalsize,
	]
	% =============================================================================
	% Panel (a) -- W(n) vs p
	% =============================================================================
	\begin{scope}
		\def\xfig{4}     % right edge of plot area (axis extends past p=1)
		\def\yfig{4}     % top edge of plot area
		\def\xc  {2}     % p = exp(-1/c)
		\def\xone{4}     % p = 1   <-- hatching and threshold line stop here
		\def\yc  {2}     % W(n) = c log n  (centred on the y-axis)

		% --- hatched regions: uniform above and below, stop at p = 1 -----------
		%  Above the threshold: theta_n(p) -> 1   --- NW diagonals (\\\\)
		\fill[pattern=north west lines] (0,\yc) rectangle (\xone,\yfig);
		%  Below the threshold: theta_n(p) -> 0   --- dashed NE diagonals (//)
		\fill[pattern=ne dashed lines]  (0,0  ) rectangle (\xone,\yc);

		% --- horizontal threshold line at c log n: spans 0 -> p = 1 ------------
		\draw[dashed,thick] (0,\yc) -- (\xone,\yc);

		% --- vertical solid segment at p = exp(-1/c) (below the threshold) -----
		\draw[thick]            (\xc,\yc) -- (\xc,0  );
		\draw[thick,fill=white] (\xc,\yc) circle (2.5pt);   % open circle

		% --- axes ---------------------------------------------------------------
		\draw[->,thick] (0,0) -- (0,\yfig+0.5) node[above] {$W(n)$};
		\draw[->,thick] (0,0) -- (\xfig+0.6,0) node[right] {$p$};

		% --- y-axis labels (symmetric around the centred threshold) -----------
		\draw[thick] (-0.08,\yc) -- (0.08,\yc);
		\node[left=2pt] at (0,\yc) {$c\log n$};       % <-- renamed from log n
		\node[left=2pt] at (0,3)   {$\omega(\log n)$};
		\node[left=2pt] at (0,1)   {$o(\log n)$};

		% --- x-axis labels ------------------------------------------------------
		\draw[thick] (\xc,-0.08) -- (\xc,0.08);
		\node[below=2pt] at (\xc  ,0) {$\exp(-1/c)$};
		\draw[thick] (\xone,-0.08) -- (\xone,0.08);
		\node[below=2pt] at (\xone,0) {$1$};

		% --- region labels ------------------------------------------------------
		\node[fill=white,inner sep=2pt] at (2, 3) {$\theta_n(p)\to 1$};
		\node[fill=white,inner sep=2pt] at (2, 1) {$\theta_n(p)\to 0$};

		% --- panel caption ------------------------------------------------------
		\node at (\xfig/2, -1.4) {$(a)$};
	\end{scope}

	% =============================================================================
	% Panel (b) -- c vs p,  curve  c = -1/ln(p)   i.e.  p = exp(-1/c)
	% =============================================================================
	\begin{scope}[xshift=5.5cm]
		\def\xfig{4}     % right edge = position of p = 1 asymptote
		\def\yfig{4}     % top edge
		\def\cval{2}     % y-coordinate of the labelled c tick

		% Re-parametrize by c (instead of by p) so the curve genuinely starts
		% at the origin.  As c -> 0+,   p = exp(-1/c) -> 0,   c -> 0,   so the
		% curve passes through (0,0) in plotting coordinates.  Sampling uniformly
		% in c also distributes samples evenly along the curve.
		\pgfmathsetmacro{\pexit}{exp(-1/\yfig)}     % top-edge crossing, x = xfig*pexit
		\def\cstart{0.01}  % c-start; gives p = exp(-100) ~ 0, y ~ 0 (visually = origin)

		% --- above-left of curve (theta_n -> 1) --- solid NE hatching ----------
		\fill[pattern=north east lines]
		(0,0)
		-- (0,\yfig)
		-- ({\xfig*\pexit},\yfig)
		-- plot[variable=\c,domain=\yfig:\cstart,samples=100]
		({\xfig*exp(-1/\c)},\c)
		-- cycle;

		% --- below-right of curve (theta_n -> 0) --- dashed NE hatching --------
		\fill[pattern=ne dashed lines]
		(0,0)
		-- plot[variable=\c,domain=\cstart:\yfig,samples=100]
		({\xfig*exp(-1/\c)},\c)
		-- (\xfig,\yfig)
		-- (\xfig,0)
		-- cycle;

		% --- the curve itself --------------------------------------------------
		\draw[very thick]
		plot[variable=\c,domain=\cstart:\yfig,samples=140]
		({\xfig*exp(-1/\c)},\c);

		% --- dashed vertical asymptote at p = 1 (x = xfig) ---------------------
		\draw[dashed,thick] (\xfig,0) -- (\xfig,\yfig+0.3);

		% --- axes ---------------------------------------------------------------
		\draw[->,thick] (0,0) -- (0,\yfig+0.5) node[above] {$c$};
		\draw[->,thick] (0,0) -- (\xfig+0.6,0) node[right] {$p$};

		% --- x-axis tick at p = 1 ----------------------------------------------
		\draw[thick] (\xfig,-0.08) -- (\xfig,0.08);
		\node[below=2pt] at (\xfig,0) {$1$};

		% --- curve / asymptote label ------------------------------------------
		\node[above=2pt] at (\xfig-1.2,\yfig-0.1) {$p = \exp(-1/c)$};

		% --- region labels -----------------------------------------------------
		\node[fill=white,inner sep=2pt] at (1.3, 3) {$\theta_n(p)\to 1$};
		\node[fill=white,inner sep=2pt] at (3, 1) {$\theta_n(p)\to 0$};

		% --- panel caption ------------------------------------------------------
		\node at (\xfig/2, -1.4) {$(b)$ Case $W(n) = c\log n$};
	\end{scope}
\end{tikzpicture}}
	\caption{%
		Diagrams showing the cases of \zcref[S]{prop:3:site:2:trivial, prop:3:site:2:critical}.
		Note that (b) depicts the boundary case in (a).
		Here, $L(n) = n$.
	}
	\label{fig:3:sitephasespace}
\end{figure}

These results identify that when using dropconnect in networks that are asymptotically deeper than they are wide, $W=O(\log{L})$, there may not exist a path that connects the input and output layers of the network.
In \zcref[S]{chapter:dropout} we characterise the evolution of the \gls{NN}'s parameters under this regime.

\subsection{Bond percolation}\label{sec:3:bond}

The dropconnect-bond percolation model is more challenging to study, as graph cuts may span across many layers at once, as illustrated in \zcref[S]{fig:3:sitecomparison}.
Site percolation only cuts the graph if an entire layer of vertices has been deleted, as in \zcref[S]{fig:3:sitecomparison}(b).
In bond percolation, there is similarly a cut if the entire layer of edges is missing, see \zcref[S]{fig:3:sitecomparison}(c), but additionally, the cut may instead extend across many layers as given in \zcref[S]{fig:3:sitecomparison}(d).

\begin{figure}[htbp]
	\centering
	\includegraphics[width=0.8\linewidth]{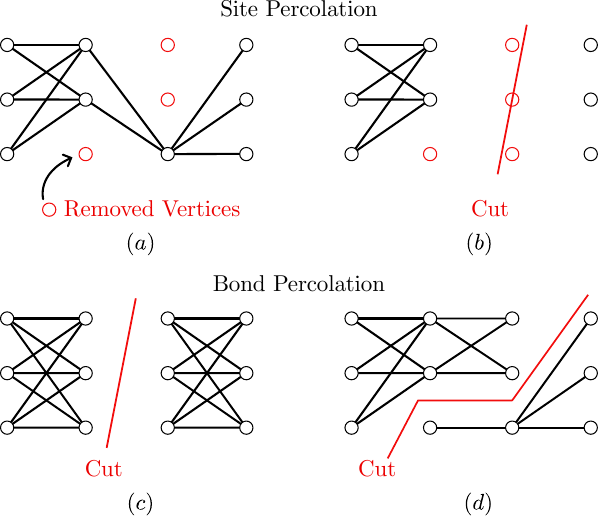}
	\caption{%
		Diagrams of bond and site percolation samples, demonstrating the differences in the occurrence of graph cuts in the models.
	}
	\label{fig:3:sitecomparison}
\end{figure}

Fortunately, we can still obtain an exact characterisation of the bond percolation function.
Computing the crossing probability is akin to path counting.
However, there are dependencies amongst the overlapping paths.
To circumnavigate this issue, we can exploit the layered structure and directness of the graph.
This allows one to obtain a recursive result by conditioning on the number of vertices reached by a path in each layer.

\begin{lemma}\label{lemma:3:bond:complete}
	Let $p\in(0,1)$, $W$, $L\in\NN$.
	Then
	\begin{equation}\label{eq:3:bond:theta_exact}
		\begin{split}
			 &
			\theta^{\text{bond}}(p, W, L)
			\\
			 &
			=
			\sum_{n_1, \ldots, n_{L+1} = 1}^W
			\prod_{\ell = 0}^{L}
			\PP{N_{\ell+1} = n_{\ell+1} \mid N_{\ell} = n_{\ell}}
		\end{split}
	\end{equation}
	where $N_\ell := | \mathcal{C}_\ell |$, $N_{\ell+1}\mid N_{\ell} = n_\ell \sim \text{Bin}(W, 1-p^{n_\ell})$ and $N_0 = W$, in which
	\begin{equation}\label{eq:3:reachable_nodes}
		\mathcal{C}_\ell :=	\set{j\in [W]  : \exists i \in [W]  \textnormal{ s.t.\ } (i,0) \leftrightarrow (j, \ell)}.
	\end{equation}
\end{lemma}

Limits of~\eqref{eq:3:bond:theta_exact} are non-trivial to analyse.
Instead, we find lower and upper bounds of $\theta$ that have analysable limits.

\begin{proposition}\label{prop:3:bond:1}
	The following bounds hold:
	$
		(
		1-p^{\frac{W^2}{4} +W\frac{\log{2}}{\log{p}}}
		)^{L+1}
		\allowbreak
		\leq
		\allowbreak
		\theta^{\text{bond}}(p, W, L)
		\allowbreak
		\leq
		\allowbreak
		{(1-p^{W^2})}^{L+1}
	$.
\end{proposition}

Observe that these bounds have a similar form to the site percolation probability in \zcref[S]{prop:3:site:1}.
Here, however, the exponent of $p$ scales with the square of the width.
There are after all on the order of $W^2$ independent filters between each layer rather than on the order of $W$ sites.
This connection is explored further in \zcref[S]{sec:3:connection}.

The following results characterise the limiting behaviour of these bounds.
In the first, we consider non-critical scaling of the topology.

\begin{proposition}[Noncritical Scaling]\label{prop:3:bond:2:trivial}
	Let $p\in(0,1)$, $L(n) = n$.
	The following then holds:
	\begin{enumerate}[label=\roman*.]
		\item If $W(n) = \omega(\sqrt{\log{n}})$, then $\theta_n^{\text{bond}}(p) \to 1$ as $n\to \infty$ and $p_c = 1$.
		\item If $W(n) = o(\sqrt{\log{n}})$, then $\theta_n^{\text{bond}}(p) \to 0$ as $n\to \infty$ and $p_c = 0$.
	\end{enumerate}
\end{proposition}

This result is analogous to \zcref[S]{prop:3:site:2:trivial} in site percolation, and states that the critical topological scaling is square root of logarithmic.
This corresponds well to our site percolation results and the intuition that there are now $W^2$ edges which can be independently removed, as opposed to the $W$ sites.
In the following result, we consider the critical probability under the critical topological scaling.

\begin{proposition}[Critical Scaling]\label{prop:3:bond:2:critical}
	Let $p\in(0,1), c>0$, $L(n) = n$.
	If $W(n)/\sqrt{\log n} \to \sqrt{c}$ as $n\to\infty$, then $p_c \in [\exp(-4/c), \exp(-1/c)]$.
	Moreover, the following then holds:
	\begin{enumerate}[label=\roman*.]
		\item If $p < \exp(-4/c)$, then $\theta_n^{\text{site}}(p) \to 1$ as $n\to\infty$.
		\item If $p > \exp(-1/c)$, then $\theta_n^{\text{site}}(p) \to 0$ as $n\to\infty$.
	\end{enumerate}
\end{proposition}

Under square root logarithmic scaling with rate $c$, this result bounds the critical probability as a function of rate $c$.
The bounds tighten as $c$ grows small or large, and shows how the critical probability varies with this scaling factor.
In the corresponding site percolation result, \zcref[S]{prop:3:site:2:critical}, the critical probability was shown to be $\exp{(-1/c)}$, equal to the upper bound in this result.

The lower bound $p_c\geq \exp(-4/c)$ is a consequence of $\frac{W^2}{4} + W \frac{\log 2}{\log p}$ exponent in \zcref[S]{eq:3:bond:theta_exact}.
This stems from the event that the number of nodes reached stays away from zero, in particular larger than $W/2$.
The specific choice of $1/2$ maximises the bound.
These results are depicted in \zcref[S]{fig:3:bondparameter-diagram}, which shows the different cases outlines in the results.

\begin{figure}[htbp]
	\centering
	\resizebox{\linewidth}{!}{\begin{tikzpicture}[
	>={Stealth[length=2.5mm]},
	line cap=round,
	font=\normalsize,
	]
	% --- coordinate constants ------------------------------------------------
	\def\xfig{4}        % right edge of plot area  (= \ytop, so plot is square)
	\def\ytop{4}        % top of upper hatched region
	\def\ymid{1.846}    % boundary line at W(n) = sqrt(log n)   (3   * 4/6.5)
	\def\xa{1.500}      % x-coord for p = exp(-4/c)  (nudged slightly left from 1.662 for label spacing)
	\def\xb{2.831}      % x-coord for p = exp(-1/c)             (4.6 * 4/6.5)

	% --- hatched regions ----------------------------------------------------
	%  Upper: theta_n(p) -> 1   solid NE diagonals
	\fill[pattern=north east lines] (0,\ymid) rectangle (\xfig,\ytop);
	%  Lower: theta_n(p) -> 0   dashed NE diagonals
	\fill[pattern=ne dashed lines]  (0,0)     rectangle (\xfig,\ymid);

	% --- boundary line at W(n) = sqrt(log n) --------------------------------
	%  left: solid  -> echoes the SOLID hatching of the upper region
	\draw[thick] (0,\ymid) -- (\xa,\ymid);
	%  middle transition zone: dotted
	\draw[thick,dash pattern=on 0.6pt off 3pt] (\xa,\ymid) -- (\xb,\ymid);
	%  right: dashed -> echoes the DASHED hatching of the lower region
	\draw[thick,dash pattern=on 3pt off 2pt]   (\xb,\ymid) -- (\xfig,\ymid);

	% --- axes ---------------------------------------------------------------
	\draw[->,thick] (0,0) -- (0,\ytop+0.6) node[above] {$W(n)$};
	\draw[->,thick] (0,0) -- (\xfig+0.6,0) node[right] {$p$};

	% --- y-axis labels ------------------------------------------------------
	\draw[thick] (-0.08,\ymid) -- (0.08,\ymid);
	\node[left=2pt] at (0,\ymid)  {$\sqrt{\log n}$};
	\node[left=2pt] at (0,3.077)  {$\omega(\sqrt{\log n})$};   % 5   * 4/6.5
	\node[left=2pt] at (0,0.923)  {$o(\sqrt{\log n})$};        % 1.5 * 4/6.5

	% --- x-axis ticks at exp(-4/c) and exp(-1/c) ----------------------------
	\draw[thick] (\xa,-0.08) -- (\xa,0.08);
	\node[below=2pt,font=\footnotesize,xshift=-2mm] at (\xa,0) {$\exp(-4/c)$};
	\draw[thick] (\xb,-0.08) -- (\xb,0.08);
	\node[below=2pt,font=\footnotesize,xshift=2mm] at (\xb,0) {$\exp(-1/c)$};
	\draw[thick] (\xfig,-0.08) -- (\xfig,0.08);
	\node[below=2pt] at (\xfig,0) {$1$};

	% --- vertical lines from x-axis ticks to mid boundary -------------------
	\draw[thick,dashed] (\xa,0) -- (\xa,\ymid);
	\draw[thick,dashed] (\xb,0) -- (\xb,\ymid);

	% --- region labels (white background to mask the hatching) --------------
	\node[fill=white,inner sep=2pt] at (2, 3.077) {$\theta_n(p)\to 1$};    % (3.25,5)   * 4/6.5
	\node[fill=white,inner sep=2pt] at (2, 0.923) {$\theta_n(p)\to 0$};    % (3.25,1.5) * 4/6.5

	% --- panel caption ------------------------------------------------------
	\node at (\xfig/2, -1.4) {$(a)$};

	% =============================================================================
	% Panel (b) -- c vs p,  two curves with no hatching in the middle strip
	% =============================================================================
	\begin{scope}[xshift=5.5cm]
		\def\xfig{4}        % right edge = position of p = 1 asymptote  (= \yfig, so square)
		\def\yfig{4}        % top edge
		\def\cmax{10}       % maximum c value plotted (mapped to top of panel)
		\def\cstart{0.01}   % c-start; gives p ~ 0, y ~ 0 (visually = origin)

		% top-edge crossings (parametrize by c so the curves start at the origin)
		\pgfmathsetmacro{\pexitB}{exp(-4/\yfig)}   % leftmost  curve, p = exp(-4/c)
		\pgfmathsetmacro{\pexitA}{exp(-1/\yfig)}   % rightmost curve, p = exp(-1/c)

		%----- LEFT of curve B (theta_n -> 1):  solid NE hatching --------------
		\fill[pattern=north east lines]
		(0,0)
		-- (0,\yfig)
		-- ({\xfig*\pexitB},\yfig)
		-- plot[variable=\c,domain=\cmax:\cstart,samples=120]
		({\xfig*exp(-4/\c)},{\c*\yfig/\cmax})
		-- cycle;

		%----- RIGHT of curve A (theta_n -> 0):  dashed NE hatching ------------
		\fill[pattern=ne dashed lines]
		(0,0)
		-- plot[variable=\c,domain=\cstart:\cmax,samples=120]
		({\xfig*exp(-1/\c)},{\c*\yfig/\cmax})
		-- (\xfig,\yfig)
		-- (\xfig,0)
		-- cycle;

		%----- the two curves themselves --------------------------------------
		\draw[very thick]
		plot[variable=\c,domain=\cstart:\cmax,samples=200]
		({\xfig*exp(-4/\c)},{\c*\yfig/\cmax});
		\draw[very thick]
		plot[variable=\c,domain=\cstart:\cmax,samples=200]
		({\xfig*exp(-1/\c)},{\c*\yfig/\cmax});

		%----- dashed vertical asymptote at p = 1 -----------------------------
		\draw[dashed,thick] (\xfig,0) -- (\xfig,\yfig+0.2);

		%----- axes -----------------------------------------------------------
		\draw[->,thick] (0,0) -- (0,\yfig+0.5) node[above] {$c$};
		\draw[->,thick] (0,0) -- (\xfig+0.6,0) node[right] {$p$};

		%----- x-axis tick at p = 1 -------------------------------------------
		\draw[thick] (\xfig,-0.08) -- (\xfig,0.08);
		\node[below=2pt] at (\xfig,0) {$1$};

		%----- curve labels (B above where it exits, A above the asymptote) ---
		\node[above=2pt, font=\footnotesize, xshift=+5mm] at ({\xfig*\pexitB},\yfig) {$p = \exp(-4/c)$};
		\node[above=2pt, font=\footnotesize, xshift=+0mm] at (\xfig,\yfig)            {$p = \exp(-1/c)$};

		%----- region labels --------------------------------------------------
		\node[fill=white,inner sep=2pt] at (1.2, 3.2) {$\theta_n(p)\to 1$};   % (1.5,5)   * 4/6.5
		\node[fill=white,inner sep=2pt] at (3.1, 0.4) {$\theta_n(p)\to 0$};   % (4.5,1.5) * 4/6.5

		%----- panel caption --------------------------------------------------
		\node at (\xfig/2, -1.4) {$(b)$ Case $W(n) = \sqrt{c\log n}$};
	\end{scope}
\end{tikzpicture}}
	\caption{Diagram depicting the cases of \zcref[S]{prop:3:bond:2:trivial, prop:3:bond:2:critical}}
	\label{fig:3:bondparameter-diagram}
\end{figure}

\begin{remark}
	\eqref{eq:3:bond:theta_exact} is a sum over all possible paths of a Markov chain, specifically, $(N_\ell)_{\ell \in \NN}$.
	If $P$ denotes its transition matrix, then $\theta(p, W, L) = 1 - {[P^{L+1}]}_{W,0}$.
	A detailed hitting time analysis of this chain looks needed to close the gap in \zcref[S]{prop:3:bond:2:critical}.
\end{remark}

\subsubsection{Monotonicity properties of \texorpdfstring{$\theta^{\text{bond}}$}{theta-bond}}

Typically, percolation functions exhibit monotonicity in $p$ as more bonds imply a greater probability of a path or giant component.
We prove such monotonic properties in the following two lemmas using graph coupling ideas.

\begin{lemma}\label{lemma:3:bond:monotonic_p}
	Let $p_1, p_2 \in(0,1)$ where $p_1\leq p_2$ and $L, W\in \NN$.
	Then $\theta^{\text{bond}}(p_1, W, L) \geq \theta^{\text{bond}}(p_2, W, L)$.
\end{lemma}

\begin{lemma}\label{lemma:3:bond:monotonic_w}
	Let $p \in (0,1), L\in \NN$ and $W_1, W_2\in\NN$ where $W_1 \leq W_2$.
	Then $\theta^{\text{bond}}(p, W_1, L) \leq \theta^{\text{bond}}(p, W_2, L)$.
\end{lemma}

\zcref[S]{lemma:3:bond:monotonic_p,lemma:3:bond:monotonic_w} follow from the intuition that if there are more bonds and thus more paths, then the probability of a crossing path is larger.
The lemmas also imply that $\theta_n(p)$ is non-increasing in $p$, which ensures the percolation function is well-behaved.

\subsection{Connecting site and bond percolation}\label{sec:3:connection}

The results in \zcref[S]{sec:3:site,sec:3:bond} showed similarities in structure, scaling, and critical probability.

The bounds on the bond percolation function in \zcref[S]{prop:3:bond:1} have a similar form to those for site percolation function in \zcref[S]{prop:3:site:1}, but with $W$ replaced by $W^2$.
\zcref[S]{prop:3:bond:2:critical, prop:3:site:2:critical} focused on this difference in scaling leading to similar critical probabilities.
For the same $p$, both site and bond models have the same number of expected edges between each layer, that is, $pW^2$.
However, the bond model has on the order of $W$ times as many distinct paths.
This leads to the square-root relationship with the topology scaling.

Instead, we may consider scaling $p$ and fixing the topology between the models.
In the following result we recover the upper bound on the bond percolation function, \zcref[S]{prop:3:bond:1} (up to $L\mapsto L+1$), using a coupling argument with site percolation.

\begin{lemma}\label{lemma:3:connection}
	$\theta^{\text{bond}}(p, W, L) \leq \theta^{\text{site}}(p^W, W, L) = (1-p^{W^2})^{L}$
\end{lemma}

The probability there are no edges between layer $\ell$ and $\ell+1$ in bond percolation is $p^{W^2}$, so the upper bound is just the probability that there are no cuts between two layers.
Similarly in \zcref[S]{prop:3:site:1}, the site percolation function is also the probability that there are no cuts between two layers.
This is connected formally in the coupling argument.
Coupling bond and site percolation is typical of percolation problems, see for example~\cite[Section~1.6]{Grimmett:1999:Percolation}.

% chktex 9
\section{Breakdown of Dropout due to Percolation}\label{chapter:dropout}

In this section, we detail the effect of percolation on the behaviour of dropout.
In particular, we show that a lack of paths across the network causes a critical breakdown in the performance of dropout in which no learning occurs; see \zcref[S]{thm:4:main}.

Reference~\cite{Senen-Cerda.ea:2025:Almost} shows that, under suitable regularity conditions, dropout is well-behaved and converges to critical points.
Specifically, for any fixed \gls{NN} topology, taking the iteration limit first allows sufficient time to overcome path-related issues induced by dropout.
In contrast, if the limit is taken in the network topology first, then with high probability no path exists across the network.
In this case, the network cannot learn, as the output becomes independent of the input.
This phenomenon is established rigorously in \zcref[S]{lemma:4:no_path}.

To bridge this gap between the order of the limits, we ask whether there is some intermediate scaling of topology and training time that exhibits a non-trivial percolation problem.
This concept of an intermediate scaling is depicted in \zcref[S]{fig:4introphasespace}, which shows that this acts as a boundary between the two trivial cases from the ordering of the limits.

\begin{figure}[htbp]
	\centering
	\resizebox{0.8\linewidth}{!}{\begin{tikzpicture}[>=Latex, font=\large]
	
	% Axes
	\draw[->, thick] (0,0) -- (6,0) node[below=8pt, font=\Large] {Network Depth};
	\draw[->, thick] (0,0) -- (0,6) ;
	\node[left, font=\Large] at (-0.3,3) {Iteration};
	
	% The curve (exponential-like growth from origin)
	\draw[ultra thick, ->]
	plot[domain=0:2.3, samples=80, smooth]
	(\x*2.2, {0.55*(exp(\x)-1)});
	
	% Labels
	\node[align=left, anchor=south west] at (0.8 ,4.5)
	{No percolation\\ problem};
	
	\node[align=left, anchor=west] at (5.2,5.3)
	{Intermediate scaling\\ exhibiting non-trivial\\ percolation problem};
	
	\node[align=left, anchor=west] at (4.5,1.2)
	{Trivial percolation\\ problem};
	
\end{tikzpicture}}
	\caption{% 
		A diagram depicting the effect of the order of limits in iterations (training time) and network depth on the percolation problem.
	}%
	\label{fig:4introphasespace}
\end{figure}

Our intermediate scaling perspective on the dropout problem constitutes one of the key insights in this section.
We believe that this scaling also has a better \gls{ML} interpretation than taking either limit first.
In practice, \glspl{NN} are only trained for finitely many steps, and with more parameters, yet ever larger \glspl{NN} may need more training.
In this way, the scaling says for how long a given sized network should be trained.

\subsection{The percolation problem}\label{sec:4:perc}

This section relies on two assumptions.
The first, \zcref[S]{ass:4:second_moment}, is weaker than the standard finite-second-moment assumption in stochastic approximation (\zcref[S]{ass:2:SO:second_moment}), while still ensuring that the gradient estimate is well behaved.
The second, \zcref[S]{ass:4:zero_updates}, asks that if there is no path across the network, then the dropout gradient is zero.
This is the critical assumption from which the results follow and is notably satisfied by \glspl{NN} without biases.

\begin{assumption}\label{ass:4:second_moment}
	$\sup_{n,t,x}\EEs{\normRs{g^{(n)}( x, \xi^{(n)}_t)}}<\infty$
\end{assumption}

\begin{assumption}\label{ass:4:zero_updates}
	If $\CC(G(F( \cdot , f))) = \emptyset$, then for all $w\in\CW$, $f\odot g(f\odot w, \cdot) = 0$ almost surely.
\end{assumption}

The following upper bounds the difference in the parameters at the end of training and their initial condition, in terms of the training time and percolation probability.

\begin{theorem}\label{thm:4:main}
	Let ${(F_n( \cdot , w^{(n)}_{T(n)}))}_{n\in\NN}$ be a sequence of deep feedforward \glspl{NN}, where each network is trained for $T(n) \in \NN$ steps, using the dropout algorithm
	\begin{equation}
		w^{(n)}_{t+1}
		=
		w^{(n)}_t
		-
		\alpha_t^{(n)} f^{(n)}_t \odot g^{(n)}(f^{(n)}_t \odot w^{(n)}_t, \xi_t^{(n)}),
	\end{equation}
	for independent and identically distributed filters $f^{(n)} \sim \lambda^{(n)}$ under \zcref[S]{ass:4:second_moment,ass:4:zero_updates}.
	Define the percolation probability as
	\begin{equation}
		\theta(n)
		:=
		\PP{\norms{\mathcal{C}(G( F_n( \cdot , f^{(n)})) )}> 0}
		.
	\end{equation}
	Then there exists an $M>0$ such that for all $n\in\NN$,
	\begin{equation}
		\EE{\lVert{w_{T(n)}^{(n)} - w_0^{(n)}}\rVert} \leq M\theta(n)\sum_{t=0}^{T(n)-1}\alpha_t^{(n)}.
	\end{equation}
\end{theorem}

\zcref[S]{thm:4:main} characterises the expected difference in norm of the parameters after being trained for $T(n)$ time steps.
In particular, it connects the dropout--percolation probability of a \gls{NN} with the behaviour of the algorithm, thus connecting the scaling of the topology in $\theta(n)$ and scaling of training time $T(n)$.
Observe that if the network is trained for too few steps, or if the time steps are too small, and $\theta(n)\to 0$ more rapidly than $\sum_{t=0}^{T(n)-1} \alpha_t^{(n)}$ grows large, then the parameters do not move from their initial condition.
This means that if scaled incorrectly, the lack of paths causes no learning to happen.

The choice of topology scaling and $T(n)$ such that $\theta(n)\sum_{t=0}^{T(n)-1} \alpha_t^{(n)} \to 0$ as $n\to\infty$ gives an intermediate curve as in \zcref[S]{fig:4introphasespace} such that the percolation problem exists.
As such, to avoid this percolation problem, the networks must be trained for $T(n)$ time steps such that $\sum_{t=0}^{T(n)-1} \alpha_t^{(n)} \gg 1/\theta(n)$.
However, this condition does not guarantee that dropout behaves well; it only guarantees the avoidance of the collapse of dropout given by this theorem.
In this sense, the condition forms a lower bound on $T(n)$.

\subsubsection{The case of dropconnect}

\zcref[S]{thm:4:main} is written in general terms, but if we specify the filter distribution and network topologies, then we can find the largest scaling of $T(n)$ such that the percolation problem is not avoided.
In \zcref[S]{cor:4:constantW} we use our percolation results from \zcref[S]{sec:3:bond} to find a closed form scaling of $T(n)$ when using dropconnect.

\begin{corollary}\label{cor:4:constantW}
	Assume the conditions of \zcref[S]{thm:4:main}.
	Let each \gls{NN} $F_n$ have depth $n$ and width $W(n)$ for $n\in\NN$, and choose filters $f_t^{(n)}$ distributed according to dropconnect with parameter $p\in(0,1)$.
	Then there exists an $M>0$ such that for all $n\in\NN$,
	\begin{equation}\label{eq:4:constantW:bound}
		\EE{\lVert{w_{T(n)}^{(n)} - w_0^{(n)}}\rVert}
		\leq
		M \exp{(-np^{{W(n)}^2})}
		\sum_{t=0}^{T(n)-1}
		\alpha_t^{(n)}.
	\end{equation}
	In particular, assuming that $np^{{W(n)}^2} \to\infty$, then the following holds:
	\begin{enumerate}[label=\roman*)]
		\item
		      Let $\rho\in[0,1)$, $c\in (0, {(1-\rho)}^{-1}) $ and $\alpha>0$.
		      If $\alpha_t^{(n)} = \alpha / {(t+1)}^\rho$ and $T(n) = O(\exp(cnp^{{W(n)}^2})) $, then $\EE{\lVert{w_{T(n)}^{(n)} - w_0^{(n)}}\rVert} \to 0$ as $n\to\infty$.
		\item
		      Let $c\in(0,1)$ and $\alpha>0$.
		      If $\alpha_t^{(n)} = \alpha/(t+1)$ and $T(n) = O(\exp(\exp(cnp^{{W(n)}^2})))$, then $\EE{\lVert{w_{T(n)}^{(n)} - w_0^{(n)}}\rVert} \to 0$ as $n\to\infty$.
	\end{enumerate}
\end{corollary}

\zcref[S]{sec:3:bond} demonstrated the existence of a percolation effect in dropconnect and identified topological scalings under which this occurs.
This corollary states that the percolation effect induced by these scalings causes an issue whereby no learning happens at all, when training with dropconnect.
The condition $np^{{W(n)}^2}\to\infty$ is namely satisfied if $p\in(0,1)$ and $W(n) = o(\sqrt{\log(n)})$ and thus the training steps $T(n)$ necessary to avoid the percolation problem grows large.

Let us also highlight the case that $W$ is constant and with common learning rate $\alpha_t = \frac{\alpha}{t+1}$.
For this case, the growth of $T(n)$ is depicted in \zcref[S]{fig:4corconstantw}.

To avoid the percolation problem in the scenario of \zcref[S]{fig:4corconstantw}, the \gls{NN} must be trained for a doubly exponential number of training steps in terms of the depth.
As such, when deciding whether to add depth to the network, one should consider that many more training steps are needed if the network is already deep.
However, quantitatively, if $p$ is small, then $p^{W^2}$ is significantly smaller and so these issues can be mitigated by the choice of~$p$.
Moreover, by \zcref[S]{cor:4:constantW}.i we see that a slower decaying learning rate allows for $T(n)$ to grow more slowly.
Also, if the network is much wider than it is deep, then there is no percolation problem as given by \zcref[S]{prop:3:bond:2:trivial}.

\begin{figure}[hbtp]
	\centering
	\includegraphics[width=0.618\linewidth]{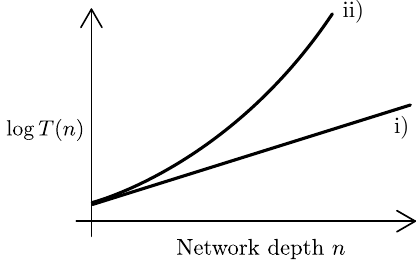}
	\caption{%
		Depiction of the growth of the training time required to avoid the percolation problem for constant width \glspl{NN} trained with dropconnect, as given by \zcref[S]{cor:4:constantW}.
		The log training time axis is used to show the difference in exponential and doubly exponential.
	}%
	\label{fig:4corconstantw}
\end{figure}

\subsubsection{Discussion}

\paragraph{Learning rates form continuous time axis.}
The term $\sum_{t=0}^{T(n)-1} \alpha_t^{(n)}$ in \zcref[S]{thm:4:main} can be thought of as a transformation of the discrete training time axis into a continuous time axis~\cite[Chapter 5]{Kushner.ea:2003:Stochastic}.
As such, the bound in \zcref[S]{thm:4:main} can be seen as a product of the transformed time and the percolation probability.

In particular, observe that with smaller percolation probabilities, parameters can move less distance within a fixed iteration budget.
This means that convergence of the training algorithm slows down.

\paragraph{\zcref[S]{thm:4:main} implies lazy training.}
\zcref[S]{thm:4:main} shows that for very deep networks, the parameters do not move far from the initialisation.
This is similar to the idea of lazy training.

Lazy training refers to training in a strongly over-parameterised scenario in which parameters move little from their initial value.
The \gls{NN} can then be approximated by a linear model around the initialisation.
Lazy training is used in practice for cheaper training steps, but it is also used in theoretical work to obtain convergence results~\cite{Mianjy.ea:2020:Convergence,Chizat.ea:2019:Lazy}.

In particular, \zcref[S]{thm:4:main} implies that for a sufficiently deep network, the probability that the process stays within an arbitrarily small region around the initialisation, during which the \gls{NN} can be approximated by a linear function, can be bounded.
We may be able to use such a linear approximation to further characterise the parameter process.

\paragraph{Gradient norm bounds are alone an incomplete picture.}
References~\cite{Senen-Cerda.ea:2025:Almost, Mohtashami.ea:2022:Masked} characterise the averages of the dropout gradient.
This is a weak characterisation of convergence and does not imply convergence of the iterates to the minimiser of the intended objective as demonstrated by our results.
In particular, the gradient estimate decays with percolation probability $\theta(n)$ as the networks grow large and thus the convergence obtained from these results may be induced by the percolation probability.

Note that the boundedness assumptions of~\cite[Proposition~7]{Senen-Cerda.ea:2025:Almost} are similar to \zcref[S]{ass:4:second_moment}, and the lack of biases and choice of ideal gradient estimate imply \zcref[S]{ass:4:zero_updates} (as discussed in \zcref[S]{sec:4:int:no_bias}).
Thus, under similar assumptions to the related work, we observe that there exists $M>0$ such that
\begin{equation}\label{eq:4:grads_to_0}
	\EE{\normR{f \odot\nabla R(f \odot w, \xi )}} \leq \theta_n(p)M,
\end{equation}
and so we clearly see that the gradients decay with the percolation function.
This corresponds to the geometry of the dropout objective becoming flat as the percolation function becomes small.

Intuitively, the dropout algorithm enforces that filtered weights are not updated, and so on average a fraction $p$ of the entries of the gradient estimate are equal to zero.
Therefore, as $p$ grows large, the magnitude of the gradient estimate decreases.
Our percolation results suggest that for network topologies susceptible to percolation effects, the fraction of zero gradients is much larger than~$p$ due to this percolation probability.

\paragraph{\zcref[S]{thm:4:main} implies convergence to a global minimiser.}
As seen in~\eqref{eq:4:grads_to_0} and discussed in the remark immediately above, under the conditions of \zcref[S]{thm:4:main}, the gradient estimates tend to 0 for all parameters $w\in\CW$ due to the percolation probability.
As such, the implied objective must be some constant with respect to the parameters $w$ (for details see~\eqref{eq:4:mod} in \zcref[S]{sec:4:int:algo}) and therefore \zcref[S]{thm:4:main} implies convergence to a trivial global minimiser.

\paragraph{\zcref[S]{cor:4:constantW} holds for original dropout.}
By replacing ${W(n)}^2\mapsto W(n)$ in \zcref[S]{cor:4:constantW}, we see that it also hold for original dropout through our results on site percolation in \zcref[S]{sec:3:site}.
This means that $T(n)$ must grow even larger to avoid the percolation problem when each neuron is filtered independently with probability $p$.

\subsection{Interpreting \texorpdfstring{\zcref[S]{ass:4:zero_updates}}{the assumptions}}\label{sec:4:interpretation}

\zcref[S]{sec:4:int:no_bias} next shows that networks without biases satisfy \zcref[S]{ass:4:zero_updates}, under standard choices of gradient estimates and a mild assumption on the activation functions.
\zcref[S]{sec:4:int:algo} instead modifies the dropout algorithm such that \zcref[S]{ass:4:zero_updates} holds, which bypasses a need for restrictions on the architecture of the \glspl{NN}.
We argue heuristically that this algorithm may perform better than unmodified dropout.
This suggests that a percolation problem also exists in unmodified dropout algorithms, even when applied to general \glspl{NN}.

\subsubsection{Networks without biases}\label{sec:4:int:no_bias}

Consider the following two assumptions:

\begin{assumption}\label{ass:4:int:no_bias}
	$F(x , w) = \sigma_{L+1} \circ A_{L+1} \circ \dots \circ \sigma_1 \circ A_1 (x)$ where $w = (A_{1}, \dots, A_{L+1})$ and $\sigma_\ell(0) = 0$ for all $\ell = 1, \dots,L+1$.
\end{assumption}

\begin{assumption}\label{ass:4:int:grad_est}
	Let $m \in \NN$ and ${((X_{t, i}, \allowbreak Y_{t, i}))}_{t \in \NN, i=1, \dots, m}$ be a sequence of random variables with $(X_{t, i}, Y_{t, i}) \sim \mu$ for all $t\in\NN$, $i=1, \dots, m$.
	Then let $I_t$ be a random subset of $\set{1, \dots, m}$ and choose
	\begin{equation}
		g(w, \xi_t)
		=
		\frac{1}{\norm{I_t}} \sum_{i \in I_t}\nabla \ell(F(w, X_{t, i}), Y_{t, i}).
	\end{equation}
\end{assumption}

The first, \zcref[S]{ass:4:int:no_bias}, specifies a class of \glspl{NN} without biases and with a mild restriction on the activation functions.
In \glspl{NN} without biases, dropout filters affect all relevant parameters.
This also has precedent since the previous work~\cite{Senen-Cerda.ea:2025:Almost} only considers \glspl{NN} without biases; and~\cite[Proposition~9]{Senen-Cerda.ea:2025:Almost} uses $\sigma(x) = x$ thus satisfying our condition on the activation functions.
We use \zcref[S]{ass:4:int:no_bias} to enforce that $F(0, w) = 0$, and to ensure that the gradient is zero when there is no path.

\zcref[S]{ass:4:int:grad_est} covers the most common choices in theory and practice.
If we have a dataset $\set{(x_i, y_i)}_{i=1}^m$ where we assume $(x_i, y_i)\sim \mu$, then we may choose $(X_{t, i}, Y_{t, i}) = (x_i, y_i)$ for $t\in\NN$, $i=1, \dots, m$.
Alternatively, if all samples are independent and identically distributed, then this covers the ideal choice in stochastic approximation theory such as in~\cite{Senen-Cerda.ea:2025:Almost,Fehrman.ea:2020:Convergence}.
The choice of distribution of $I_t$ allows one to model mini-batching.

To prove that \zcref[S]{ass:4:zero_updates} follows from \zcref[S]{ass:4:int:no_bias,ass:4:int:grad_est}, we require the following \zcref[S]{lemma:4:no_path}.
This lemma says that if there is no path across a \gls{NN}, then its output does not depend on its input.
We have referenced this lemma throughout the work because it formalises the intuition that no paths mean no flow of information across the \gls{NN}.

\begin{lemma}\label{lemma:4:no_path}
	Let $w\in\CW$, if $\CC(G(F( \cdot , w))) = \emptyset$, then $F(x , w) = F(0, w)$ for all $x\in\CX$.
\end{lemma}

The following \zcref[S]{lemma:4:no_bias_works} confirms that \glspl{NN} without biases of the type in \zcref[S]{ass:4:int:no_bias}, trained with standard choices of gradient estimates as in \zcref[S]{ass:4:int:grad_est}, satisfy \zcref[S]{ass:4:zero_updates} and thus the results of \zcref[S]{thm:4:main}.

\begin{lemma}\label{lemma:4:no_bias_works}
	If \zcref[S]{ass:4:int:no_bias,ass:4:int:grad_est} hold, then \zcref[S]{ass:4:zero_updates} holds.
\end{lemma}

\subsubsection{Modified dropout}\label{sec:4:int:algo}

We may alternatively enforce \zcref[S]{ass:4:zero_updates} via the learning algorithm itself.
We can require that whenever there is no path across the network, one does not update the weights.
By doing this, fewer assumptions are needed on the network architecture, and the networks may include biases and any type of activation function.

The modification sets the gradient estimate to zero at step $t$ if the filter samples at $t$ yield no path across the network.
This results in a different filter distribution.
Consider filters~$f$, then the modified filters are given by $\tilde{f} := \1{E}f$ where $E = \set{\norm{\mathcal{C}(G(F( \cdot , f)))} >0}$.
Since this is a modification of the filters, we also require no further assumptions on $g( \cdot , \xi)$ (other than \zcref[S]{ass:4:second_moment}).

Note that changing the dropout filters changes the dropout objective.
Specifically, this modified algorithm tries to minimise a different objective function.
We argue next that, philosophically, the modified objective is better than the original objective.
If this modified algorithm indeed tries to minimise a better objective and still exhibits the percolation problem, this suggests that a percolation problem also exists for the unmodified algorithm.
However, this percolation problem may not exhibit itself in the same manner as \zcref[S]{thm:4:main} because the unmodified algorithm may still update the parameters even if there is no learning.
We do not have a proof for these claims, but we can argue them by looking at the objectives.

Consider dropout with filters $f\sim \lambda$ and ideal \gls{SGD} gradient estimate $g(w, \xi) = \nabla \ell (F(X, w), Y)$.
We can decompose the dropout objective into two parts depending on the event $E$.
Let $\theta(\lambda) := \PP{E}$ be the percolation probability under filters $f$, then
\begin{align}
	D(w)
	 &
	=
	\EE{\ell(F(X,{f}\odot w), Y)}
	\\
	 &
	=
	\PP{E} \EE{\ell(F(X, {f}\odot w), Y)\mid E}
	\nonumber \\
	 &
	\phantom{=}
	+
	\PP{\overline{E}} \EE{\ell(F(X, {f}\odot w), Y) \mid \overline{E}}
	\\
	 &
	=
	\theta(\lambda)D_{\text{path}}(w)
	\nonumber \\
	 &
	\phantom{=}
	+
	(1-\theta(\lambda))D_{\text{no path}}(w),
\end{align}
with $D_{\text{path}}(w), D_{\text{no path}}(w)$ defined accordingly.
By \zcref[S]{lemma:4:no_path},
\begin{align}
	D_{\text{no path}}(w)
	 &
	:=
	\EE{\ell(F(X, {f}\odot w), Y)\mid \overline{E}}
	\\
	 &
	=
	\EE{\ell(F(0, {f}\odot w), Y)\mid \overline{E}}
\end{align}
which is independent of the input data and therefore cannot be a \textquote{good} objective.
If the \gls{NN} $F$ is deep and thus $\theta(\lambda)$ is small, then $D(w)$ is dominated by this poor objective $D_{\text{no path}}(w)$ and so the majority of the geometry (including the minimisers) is not influenced by the input data.
On the other hand, the objective $\tilde{D}(w)$ for the modified filters $\tilde{f}$ can be deconstructed as such,
\begin{align}\label{eq:4:mod}
	\tilde{D}(w)
	 &
	=
	\EE{\ell(F(X, \tilde{f}\odot w), Y)}
	\nonumber \\
	 &
	=
	\PP{E} \EE{\ell(F(X, \tilde{f}\odot w), Y)\mid E}
	\nonumber \\
	 &
	\phantom{=}
	+
	\PP{\overline{E}} \EE{\ell(F(X, \tilde{f}\odot w), Y)\mid \overline{E}}
	\nonumber \\
	 &
	=
	\theta(\lambda) \EE{\ell(F(X, f \odot w), Y)\mid {E}}
	\nonumber \\
	 &
	\phantom{=}
	+
	(1-\theta(\lambda))\EE{\ell(F(X, 0), Y)}
	\nonumber \\
	 &
	=
	\theta(\lambda) D_{\text{path}}(w)
	+
	\text{ Constant}.
\end{align}
Therefore, the modified objective is similar, but is only dependent on $D_{\text{path}}(w)$, which contains all the information from the input \emph{and} output data.
Therefore, this objective looks closer to the intended objective with the original filters, but removes the case that there is no path, which we do not want to optimise for.

Ideally, we would like to move beyond \zcref[S]{ass:4:zero_updates} and show that a percolation problem exists for general \glspl{NN} with biases and any dropout algorithm.
This requires future research, as new mathematical tools are required for this.

\section{Conclusion}

We proposed and explored the problem of percolation in training deep \glspl{NN} with dropout.
We established on which \gls{NN} topologies dropout and dropconnect filters exhibit percolation behaviour, whereby no paths connect the input and output layers of the \gls{NN}.
We showed that this percolation behaviour can cause a critical breakdown in the performance of dropout in \glspl{NN} without biases, in which no learning is achieved by training the \gls{NN}.
Next, we discussed that this problem should exist for general \glspl{NN} with biases.
The lack of paths causes the dropout objective to be weakly dependent on the input data and thus does not capture the goal of supervised learning.

A few aspects of this research are left open to future work.
In \zcref[S]{sec:3:bond}, the gap can be closed in the critical probability of the dropconnect--percolation problem in the critically scaled regime.
Building on \zcref[S]{chapter:dropout}, an exploration into the performance of modified dropout would give insight into the heuristics used in \zcref[S]{sec:4:int:algo}.
Additionally, a numerical study comparing (modified) dropout would be valuable to the applicability of \zcref[S]{chapter:dropout} to real world \glspl{NN}, and for establishing whether the percolation effect exists in practical network architectures.

%We close with a word of advice.
%This work shows that one should consider the interplay between the hyperparameters depth, width, and dropout probability.
%If your neural network with dropout is not sufficiently wide, it may experience slow convergence because of percolation.

\bibliographystyle{plain}
\bibliography{candlekeep}

\onecolumngrid{}

\appendix

\section{Auxiliary lemmas}

We use the following two lemmas as conditions for the convergence of the site percolation function.

\begin{lemma}\label{lemma:3:proofs:np0}
	Let $p\in (0,1)$.
	If $np^{W(n)}\to 0$ as $n\to\infty$, then ${(1-p^{W(n)})}^n \to 1$ as $n\to\infty$.
\end{lemma}

\begin{proof}[Proof of \zcref{lemma:3:proofs:np0}]
	Let $p\in(0,1)$.
	Then by assumption $p^{W(n)}\to 0$ and thus $\log(1-p^{W(n)}) = -p^{W(n)} + O(p^{2W(n)})$.
	Consequently,
	\begin{equation}
		{(1-p^{W(n)})}^n
		=
		\exp(n\log(1-p^{W(n)}))
		=
		\exp(-np^{W(n)} + O(np^{2W(n)}))
		\to
		1
		.
	\end{equation}
\end{proof}

\begin{lemma}\label{lemma:3:proofs:npinfty}
	Let $p\in (0,1)$.
	If $np^{W(n)}\to \infty$ as $n\to\infty$, then ${(1-p^{W(n)})}^n \to 0$ as $n\to\infty$.
\end{lemma}

\begin{proof}[Proof of \zcref{lemma:3:proofs:npinfty}]
	Let $p\in(0,1)$.
	Then
	\begin{equation}
		{(1-p^{W(n)})}^n = \exp{(n\log(1-p^{W(n)}))} \leq \exp(-np^{W(n)}) \to 0.
	\end{equation}
\end{proof}

\begin{remark}
	The conditions in \zcref[S]{lemma:3:proofs:np0,lemma:3:proofs:npinfty} are extendable to the case that $p=p(n)$ scales with the depth.
	In fact, these conditions show a direct relationship between the scaling of $p(n)$ and $W(n)$.
\end{remark}

\section{Remaining proofs}\label{app:b}

\begin{proof}[Proof of \zcref{prop:3:site:1}]
	Since each vertex is fully connected to every vertex in the next layer, as long as there is at least 1 vertex in each layer, there is a path across the network.
	As such, we see that the probability that there is a path across the network is just the probability that there are no vertical cuts in the graph.
	The probability that there is at least one vertex in layer $\ell$ is equal to $1-p^W$.
	Since there are $L$ hidden layers of vertices, $\theta^{\text{site}}(p, W, L) = {(1-p^W)}^L$.
	Note that there are $L+2$ total layers, but the first and last layers do not have vertices removed.
\end{proof}

\begin{proof}[Proof of \zcref{prop:3:site:2:trivial}]
	This proof follows from \zcref[S]{lemma:3:proofs:np0, lemma:3:proofs:npinfty}.
	We just need to check the conditions on $np^{W(n)}$.
	Begin by observing that
	\begin{equation}
		\label{eq:prop:3:site:proof}
		np^{W(n)}
		=
		\exp
		\Bigl(
		\log(n)
		\Bigl(
			1
			+
			\frac{W(n)}{\log n} \log p
			\Bigr)
		\Bigr)
		.
	\end{equation}
	We use this to tackle the two cases next.
	\begin{itemize}
		\item If $W(n) = \omega(\log n)$, then $W(n)/\log(n)\to \infty$.
		      Thus, $np^{W(n)} \to 0$ as $n\to\infty$ because $\log(p)<0$.
		      By \zcref[S]{lemma:3:proofs:np0}, for all $p\in(0,1)$, $\theta_n^{\text{site}}(p) \to 1$ as $n\to\infty$, and so $p_c = 1$ as $\theta_n^{\text{site}}(1) = 0$.
		\item If $W(n) = o(\log n)$, then $W(n)/\log(n)\to 0$.
		      Therefore, $np^{W(n)} \to \infty$ as $n\to\infty$.
		      By \zcref[S]{lemma:3:proofs:npinfty}, for all $p\in(0,1)$, $\theta_n^{\text{site}}(p) \to 0$ as $n\to\infty$, and so $p_c = 0$ as $\theta_n^{\text{site}}(0) = 1$.
	\end{itemize}
	This concludes the proof.
\end{proof}

\begin{proof}[Proof of \zcref{prop:3:site:2:critical}]
	This proof also follows from \zcref[S]{lemma:3:proofs:np0, lemma:3:proofs:npinfty}.
	In what follows, let $p\in(0,1)$, $c>0$.

	Observe that if $W(n)/\log n \to c$, then $1 + \log(p)W(n)/\log(n) \to 1 + c\log(p)$.
	This implies that:
	\begin{itemize}
		\item
		      If $1+c\log(p) < 0$, then $np^{W(n)}\to 0$ by~\eqref{eq:prop:3:site:proof}.
		      Hence, $\theta_n^{\text{site}}(p) \to 1$ by \zcref[S]{lemma:3:proofs:np0}.
		\item
		      If $1+c\log(p) > 0$, then $np^{W(n)}\to \infty$ by~\eqref{eq:prop:3:site:proof}.
		      Hence, $\theta_n^{\text{site}}(p) \to 0$ by \zcref[S]{lemma:3:proofs:npinfty}.
	\end{itemize}
	Combining these two facts gives that the critical probability is attained at $1+c\log(p) = 0$.
	Consequently, $p_c = \exp(-1/c)$.
	This proves parts i and ii.

	To prove part iii, observe that if $1+c\log(p)=0$ and $W(n) = c \log{n} + o(1)$, then $np^{W(n)} \to 1$.
	Hence,
	\begin{equation}
		\theta_n^{\text{site}}(p)
		=
		{(1-p^{W(n)})}^n
		=
		\exp(-np^{W(n)} + O(np^{2W(n)}))
		\to
		\exp(-1)
		.
	\end{equation}
	That is it.
\end{proof}

\begin{proof}[Proof of \zcref{lemma:3:bond:complete}]
	Let $\mathcal{C}_\ell$ be as in \zcref[S]{eq:3:reachable_nodes}, and $N_\ell := \norm{\mathcal{C}_\ell}$.
	By conditioning on the random vector $N=(N_1, \ldots, N_{L+1})$ and observing that $N_{\ell+1}$ depends only on $N_\ell$, we find that
	\begin{align}
		 &
		\theta^{\text{bond}}(p, W, L) = \PP{\norms{\mathcal{C}_{L+1}} >0}                                                                                           \\
		 &
		=\sum_{n_1, \ldots, n_{L+1} = 0}^W \PP{\norms{\mathcal{C}_{L+1}} >0 \mid N_1 = n_1, \ldots, N_{L+1} = n_{L+1}}\PP{ N_1 = n_1, \ldots, N_{L+1} = n_{L+1}}    \\
		 &
		= \sum_{n_1, \ldots, n_{L+1} = 0}^W \1{\set{n_{L+1}>0}}\PP{ N_1 = n_1, \ldots, N_{L+1} = n_{L+1}}                                                           \\
		 &
		= \sum_{n_1, \ldots, n_{L+1} = 0}^W \1{\set{n_{L+1}>0}}\PP{ N_{L+1} = n_{L+1} \mid N_1 = n_1, \ldots, N_{L} = n_{L}} \PP{ N_1 = n_1, \ldots, N_{L} = n_{L}} \\
		 &
		= \sum_{n_1, \ldots, n_{L+1} = 0}^W \1{\set{n_{L+1}>0}}\PP{ N_{L+1} = n_{L+1} \mid  N_{L} = n_{L}} \PP{ N_1 = n_1, \ldots, N_{L} = n_{L}}                   \\
		 &
		=  \sum_{n_1, \ldots, n_{L+1} = 0}^W \1{\set{n_{L+1}>0}} \prod_{\ell = 0}^{L} \PP{ N_{\ell+1} = n_{\ell+1} \mid  N_{\ell} = n_{\ell}}.
	\end{align}
	Suppose now that $N_\ell = n_\ell$.
	Then for each vertex in the $(\ell+1)$st layer, there are $n_\ell$ possible edges that could connect it to a path.
	It is thus connected with probability $1-p^{n_\ell}$ independently of every other vertex in its layer.
	Then the number of vertices in this layer connected to the previous layer is the sum of $W$ independent Bernoulli trials each with success probability $1-p^{n_\ell}$.
	Consequently, $N_{\ell+1} \mid N_{\ell} = n_\ell \sim \text{Bin}(W, 1-p^{n_\ell})$.
	We also know that if $n_\ell = 0$, then $1-p^{n_\ell }= 0$ and thus the sum over $n_\ell$ need only run from $1$ to $W$.
\end{proof}

\begin{proof}[Proof of \zcref{prop:3:bond:1}]
	Let $\mathcal{C}_\ell$ be defined as in \zcref[S]{eq:3:reachable_nodes}, i.e., the set of vertices in layer $\ell$ that are connected to any vertex in the input layer.

	We begin by proving the upper bound.
	Observe that
	\begin{align}
		\label{eq:proof:recursion}
		\theta^{\text{bond}}(p, W, L) = \PP{\norm{\mathcal{C}_{L+1}} > 0}
		 &
		= \PP{\norm{\mathcal{C}_{L+1}} > 0 \mid \norm{\mathcal{C}_{L}} > 0} \PP{\norm{\mathcal{C}_{L}} > 0}
		\\
		 &
		= \PP{\norm{\mathcal{C}_{L+1}} > 0 \mid \norm{\mathcal{C}_{L}} > 0}\theta^{\text{bond}}(p, W, L-1).
	\end{align}
	Conditioning on the number of vertices in the $\ell$th layer, one finds
	\begin{align}
		\PP{\norm{\mathcal{C}_{\ell + 1} }> 0 \mid \norm{\mathcal{C}_{\ell} }> 0}
		 &
		= \sum_{k=0}^W \PP{\norm{\mathcal{C}_{\ell + 1} }> 0 \mid \norm{\mathcal{C}_{\ell} }> 0, \;  \norm{\mathcal{C}_{\ell} } = k} \PP{ \norm{\mathcal{C}_{\ell} } = k \mid  \norm{\mathcal{C}_{\ell} }>0} \\
		 &
		= \sum_{k=1}^W \PP{\norm{\mathcal{C}_{\ell + 1} }> 0 \mid \norm{\mathcal{C}_{\ell} } = k} \PP{ \norm{\mathcal{C}_{\ell} } = k \mid  \norm{\mathcal{C}_{\ell} }>0}                                    \\
		 &
		= \sum_{k=1}^W (1-p^{kW}) \PP{ \norm{\mathcal{C}_{\ell} } = k \mid  \norm{\mathcal{C}_{\ell} }>0}                                                                                                    \\
		 & \leq \sum_{k=1}^W (1-p^{W^2}) \PP{ \norm{\mathcal{C}_{\ell} } k \mid  \norm{\mathcal{C}_{\ell} }>0}                                                                                               \\
		 & = 1-p^{W^2}.
	\end{align}
	It therefore follows from~\eqref{eq:proof:recursion} that $\theta^{\text{bond}}(p, W, L) \leq {(1-p^{W^2})}^{L+1}$.
	Note that $L$ is the number of hidden layers.
	There are thus $L+1$ layers of edges between the $L+2$ total layers of vertices.
	This means that $\theta^{\text{bond}}(p, W, 0) = \PP{\norm{\mathcal{C}_1} >0}$.

	We next prove the lower bound.
	We use a Chernoff bound to look at the event that the number of reachable nodes stays away from 0.
	As such, we focus on the event that at least half of all nodes are reached:
	\begin{align}
		\theta^{\text{bond}}(p, W, L) = \PP{\norm{\mathcal{C}_{\ell + 1} }> 0} & = \PP{\norm{\mathcal{C}_{\ell + 1} }> 0, \dots, \norm{\mathcal{C}_{1} }> 0} \geq \PP{\norm{\mathcal{C}_{\ell + 1} }> W/2, \dots, \norm{\mathcal{C}_{1} }> W/2}
		\\  &
		=
		\prod_{\ell = 0}^L
		\PP{\norm{\mathcal{C}_{\ell + 1} }> W/2 \mid \norm{\mathcal{C}_{\ell} }> W/2}
		.
		\label{eq:proof:bond:lower}
	\end{align}
	Here, we have exploited the Markov property.
	Let $n$ such that $W/2 < n \leq W$.
	From \zcref[S]{lemma:3:bond:complete} we know that $\norm{\mathcal{C}_{\ell + 1} } \mid \norm{\mathcal{C}_{\ell}}=n \sim \text{Bin}(W, 1-p^{n})$.
	We now bound
	\begin{equation}
		\PP{\norm{\mathcal{C}_{\ell + 1} }> W/2 \mid \norm{\mathcal{C}_{\ell} } = n} = \PP{\text{Bin}(W, 1-p^{n}) > W/2} = 1 - \PP{\text{Bin}(W, p^{n}) > W/2}
	\end{equation}
	We use the Chernoff bound~\cite[Thm.~2.1]{Mulzer:2018:Five}
	\begin{equation}
		\PP{\text{Bin}(W, p^{n}) > W/2} \leq \exp(-W D_{\text{KL}}(1/2 \,\|\, p^n)),
	\end{equation}
	where we may bound
	\begin{equation}
		D_{\text{KL}}(1/2 \,\|\, p^n) = -\frac{n}{2}\log{p} - \log{2} - \frac{1}{2} \log{1-p^n} \geq -\frac{W}{4}\log{p} - \log{2},
	\end{equation}
	since $\log(1-p^n) < 0$ and $n > W/2$.
	Applying the bound yields
	\begin{equation}
		\PP{\text{Bin}(W, p^{n}) > W/2} \leq \exp(\frac{W^2}{4}\log{p} + W \log(2)) = p^{W^2/4 + W\log(2)/\log(p)}
		.
	\end{equation}
	Because this holds for any $W/2 < n \leq W$, we may conclude that
	\begin{equation}
		\PP{\norm{\mathcal{C}_{\ell + 1} }> W/2 \mid \norm{\mathcal{C}_{\ell} }> W/2} \geq 1 - p^{W^2/4 + W\log(2)/\log(p)}
	\end{equation}
	for $\ell = 1,\dots, L$.
	Combining this with~\eqref{eq:proof:bond:lower} yields $\theta^{\text{bond}}(p, W, L) \allowbreak \geq \allowbreak {(1 - p^{W^2/4 + W\log(2)/\log(p)})}^{L+1}$.
\end{proof}

\begin{proof}[Proof of \zcref{prop:3:bond:2:trivial}]
	This proof follows from \zcref[S]{prop:3:site:2:trivial} after a transformation of $W$.
	In what follows, let $p\in(0,1)$ and $L(n) = n$.

	First, we prove the supercritical case; meaning, assume that $W(n) = \omega(\sqrt{\log n})$.
	Define $\tilde{W}(n) := {W(n-1)}^2 / 4 + (\log{2} / \log{p}) W(n-1)$.
	Note that $\tilde{W}(n) = \omega(\log(n))$.
	Consequently, using the lower bound of \zcref[S]{prop:3:bond:1},
	\begin{equation}
		\theta^{\text{site}}_n
		\geq
		{(1-p^{ {W(n)}^2 / 4 + (\log{2} / \log{p}) W(n)})}^{n+1}
		=
		{(1-p^{\tilde{W}(n+1)})}^{n+1}
		\to
		1
	\end{equation}
	as $n\to\infty$, where the limit is given by \zcref[S]{prop:3:site:2:trivial}.i.

	Next, we prove the subcritical case; meaning, assume that $W(n) = o(\sqrt{\log n})$.
	Define $\tilde{W}(n) := {W(n-1)}^2$.
	Note that $\tilde{W}(n) = o(\log(n))$.
	Using the upper bound of \zcref[S]{prop:3:bond:1}, we have
	\begin{equation}
		\theta^{\text{site}}_n
		\leq
		{(1-p^{{W(n)}^2})}^{n+1}
		=
		{(1-p^{\tilde{W}(n+1)})}^{n+1}
		\to
		0
	\end{equation}
	as $n\to\infty$, where the limit is given by \zcref[S]{prop:3:site:2:trivial}.ii.
\end{proof}

\begin{proof}[Proof of \zcref{prop:3:bond:2:critical}]
	This proof follows the same structure as the equivalent site percolation result \zcref[S]{prop:3:site:2:critical} after transformation of $W(n)$ in the same way as in the proof of \zcref[S]{prop:3:bond:2:trivial}.
	In what follows, let $c>0$, $L(n)=n$, and assume that $W(n)/\sqrt{\log{n}} \to \sqrt{c}$ as $n \to \infty$.

	We first deal with the case $p < \exp(-4/c)$.
	Let $\tilde{W}(n):= {W(n-1)}^2 / 4 + (\log{2} / \log{p}) W(n-1)$, and note that
	\begin{equation}
		1
		+
		\frac{\tilde{W}(n)}{\log{n}}\log{p}
		=
		1
		+
		\frac{{W(n-1)}^2}{4\log(n)}\log(p)
		+
		\frac{W(n-1)}{\log{n}}\log{2}
		\to
		1
		+
		\frac{c\log{p}}{4}
		<
		0
	\end{equation}
	as $n\to\infty$.
	Use~\eqref{eq:prop:3:site:proof} with $W(n)$ replaced by $\tilde{W}(n)$ to next conclude that $np^{\tilde{W}(n)} \to 0$ as $n\to\infty$.
	Hence, by \zcref[S]{lemma:3:proofs:np0}, $\theta_n^{\text{site}}(p) \to 1$.

	Next, consider the case $p > \exp(-1/c)$.
	The result follows for this case from \zcref[S]{prop:3:site:2:critical}.ii after transforming $\tilde{W}(n) := {W(n-1)}^2$.

	Together, the two cases imply that $p_c \in [\exp(-4/c), \exp(-1/c)]$.
\end{proof}

\begin{proof}[Proof of \zcref{lemma:3:bond:monotonic_p}]
	We can prove this with a coupling argument.
	Let $E$ denote the set of all edges in a $W\times L$ rectangular layered network and $V$ the set of vertices.
	Consider two coupled graphs $\hat{G_1} = (V, E_1), \hat{G_2} = (V, E_2)$ where $\hat{G_2}\sim G(p_2, W, L)$ and $\hat{G_1}$ is constructed as follows: suppose that for each edge $e\in E_2$ we have $e\in E_1$, then for each remaining edge $e\in E\setminus E_2$ we say that $e\in E_1$ independently with probability $1-\frac{p_1}{p_2}$.

	\begin{figure}[hbtp]
		\centering
		\includegraphics[width=0.618\linewidth]{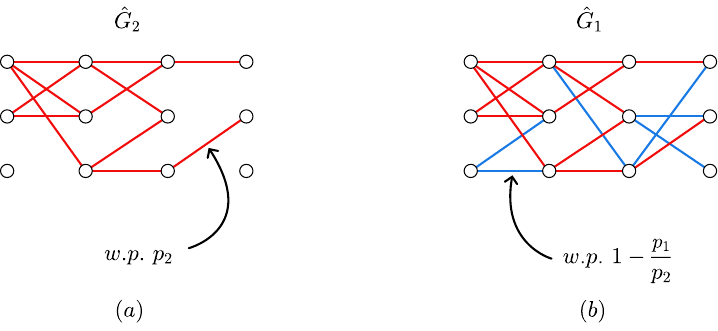}
		\caption{An example of the coupling in the proof of \zcref[S]{lemma:3:bond:monotonic_p} with $W=3, L=2$.}
	\end{figure}

	Let $e\in E$.
	Observe that
	\begin{align}
		\PP{e\in E_1}
		 &
		=
		\PP{ e \in E_1 \mid e \in E_2 }
		\PP{ e \in E_2 }
		+
		\PP{ e \in E_1 \mid e \not\in E_2 }
		\PP{ e \not\in E_2 }
		\\
		 &
		= 1\cdot (1-p_2) + \Bigl(1-\frac{p_1}{p_2}\Bigr)p_2 = 1-p_1
	\end{align}
	and therefore each edge belongs to the graph $\hat{G_1}$ independently with probability $1-p_1$, and thus it has distribution $G(p_1, W, L)$.
	To verify the independence, let $e \neq f$ and observe that indeed
	\begin{align}
		\PP{ e, f \in E_1 }
		 &
		=
		\PP{ e, f \in E_1 \mid e, f \in E_2 }
		\PP{ e, f \in E_2 }
		+
		\PP{ e, f \in E_1 \mid e \in E_2, f \not\in E_2 }
		\PP{ e \in E_2, f \not\in E_2 }
		+
		\cdots
		\\
		 &
		=
		1 \cdot {(1-p_2)}^2
		+
		2 \Bigl( 1 - \frac{p_1}{p_2} \Bigr) \cdot p_2 (1 - p_2)
		+
		{\Bigl( 1 - \frac{p_1}{p_2} \Bigr)}^2 \cdot p_2^2
		\\
		 &
		=
		{(1 - p_1)}^2
		=
		\PP{ e \in E_1 }
		\PP{ f \in E_1 }
		.
	\end{align}
	By construction, $\hat{G_2} \subset  \hat{G_1}$ and thus $ \mathcal{C}(\hat{G_2}) \subset \mathcal{C}(\hat{G_1})$.
	We therefore conclude that
	\begin{align}
		\theta^{\text{bond}}(p_1, W, L) = \PP{\norms{\mathcal{C}( \hat{G_1})} >0 } \geq \PP{\norms{\mathcal{C}( \hat{G_2})} >0} = \theta^{\text{bond}}(p_2, W, L).
	\end{align}
	Note that each edge is given two chances to be added to $\hat{G_1}$ and at each chance it is added with some probability, independently of every other edge.
	As such, each edge is indeed added independently.
\end{proof}

\begin{proof}[Proof of \zcref{lemma:3:bond:monotonic_w}]
	We can prove this with another coupling argument, the idea being to embed the smaller graph into the larger graph.

	Let $F_i$ denote the set of all edges in a $W_i\times L$ rectangular layered network and $V_i$ the set of vertices.
	Consider two coupled graphs $\hat{G_1} = (V_1, E_1), \hat{G_2} = (V_2, E_2)$ where $\hat{G_1} \sim G(p, W_1, L)$ and $\hat{G_2}$ is constructed as follows: suppose that for each edge $e\in E_1$ we have $e \in E_2$, then for each remaining edge $e \in F_2 \setminus F_1$ we say that $e\in E_2$ with probability $p$ independently for all edges.
	Consequently, $\hat{G_2}$ has distribution $G(p, W_2, L)$.
	An example of this coupling can be found in \zcref[S]{fig:3couplew}.

	\begin{figure}[htbp]
		\centering
		\includegraphics[width=0.618\linewidth]{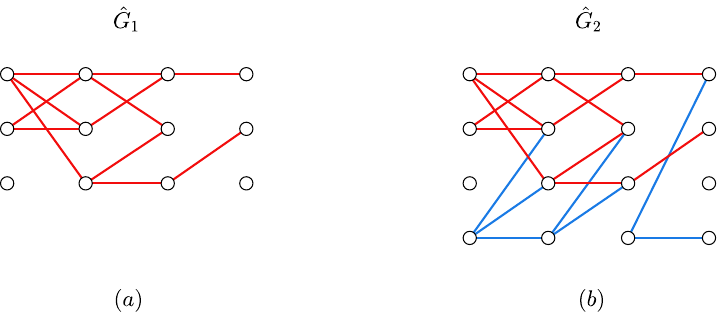}
		\caption{% 
			An example of the coupling in the proof of \zcref[S]{lemma:3:bond:monotonic_w}, with $W_1=3,  W_2 = 4$, and $L=2$.
		}%
		\label{fig:3couplew}
	\end{figure}

	By construction, $\hat{G_1} \subset \hat{G_2}$, and therefore $\mathcal{C}(\hat{G_1}) \subset \mathcal{C}(\hat{G_2})$.
	Consequently,
	\begin{align}
		\theta^{\text{bond}}(p, W_1, L) = \PP{\norms{\mathcal{C}( \hat{G_1})} >0 } \leq \PP{\norms{\mathcal{C}( \hat{G_2})} >0} = \theta^{\text{bond}}(p, W_2, L).
	\end{align}
\end{proof}

\begin{proof}[Proof of \zcref{lemma:3:connection}]
	Let $\hat{G}^{\text{bond}} \sim G^{\text{bond}}(p, W, L)$.
	We now construct a graph $\hat{G}^{\text{site}}$ from $\hat{G}^{\text{bond}}$ as follows: for all $v = {(i, \ell)}\in V^{\text{hidden}} := \set{1, \dots, W} \times\set{1, \dots, L}$, we say that
	\begin{equation}
		v \in U \;\;\;\;\;\text{ if }\;\;\;\;\; \norm{\set{{(j, \ell - 1)}  : j \in [W], \; ({(j, \ell - 1)},{(i, \ell)}) \in \hat{E}^{\text{bond}} }}=0.
	\end{equation}
	Let $H$ be the $W\times L$ Rectangular Layered Network, then define $\hat{G}^{\text{site}}$ to be the induced subgraph of $H= (V, E)$ over vertices $V\setminus U$.
	An example of these couplings can be found in \zcref[S]{fig:3:couplesitebond}.

	\begin{figure}[hbtp]
		\centering
		\includegraphics[width=0.618\linewidth]{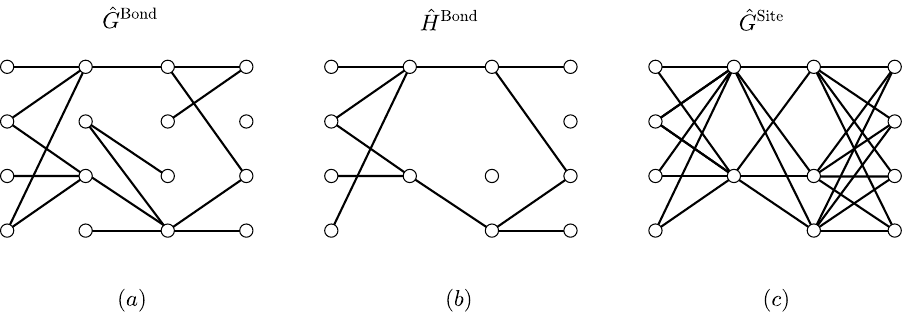}
		\caption{%
			An example of the couplings detailed in the proof of \zcref[S]{lemma:3:connection}.
			Observe that the degree 0 vertex in $\hat{H}^{\text{Bond}}$ is not deleted because it has indegree of 1 in $\hat{G}^{\text{Bond}}$, and thus the vertex belongs to $\hat{G}^{\text{Site}}$.
			This occurs because the bound in \zcref[S]{lemma:3:connection} is loose.
		}%
		\label{fig:3:couplesitebond}
	\end{figure}

	First, observe that for all $v = {(i, \ell)}\in V^{\text{hidden}}$,
	\begin{align}
		\PP{v \in \hat{V}^{\text{site}}}
		 &
		=
		\PP{\norm{\set{{(j, \ell - 1)}   : j \in [W], \; ({(j, \ell - 1)},{(i, \ell)}) \in \hat{E}^{\text{bond}} }}>0}
		\\
		 &
		=
		1 - \PP{\norm{\set{{(j, \ell - 1)}   : j \in [W], \; ({(j, \ell - 1)},{(i, \ell)}) \in \hat{E}^{\text{bond}} }}=0}
		\\
		 &
		\overset{(\star)}{=} 1 - \prod_{j=1}^W\PP{({(j, \ell - 1)},{(i, \ell)}) \not \in \hat{E}^{\text{bond}} }
		\\
		 &
		=
		1 - p^W,
	\end{align}
	where in $(\star)$ we used the independence of edges in $\hat{G}^{\text{bond}}$.
	Next, note that $v \in \hat{V}^{\text{site}}$ independently of any other vertex.
	Only edges pointing into a vertex matter and so if ${(i, \ell)} \neq {(i', \ell')}$ then the sets of edges pointing into these vertices are disjoint and thus independent due to the independence of edges in $\hat{G}^{\text{bond}}$.
	Consequently, for all $v \in V^{\text{hidden}}$, $v \in \hat{V}^{\text{site}}$ independently with probability $1-p^W$.
	In conclusion, $\hat{G}^{\text{site}} \sim G^{\text{site}}(p^W, W, L)$.

	We construct another graph $\hat{H}^{\text{bond}}$, which is the induced subgraph of $\hat{G}^{\text{bond}}$ over the vertices $\hat{V}^{\text{bond}}\setminus U$.
	Observe that this subgraph is exclusively missing vertices that were not connected to the previous layer and thus $\mathcal{C}(\hat{G}^{\text{bond}}) = \mathcal{C}(\hat{H}^{\text{bond}})$.
	By construction, we have that $\hat{H}^{\text{bond}} \subset \hat{G}^{\text{site}}$ and thus $\mathcal{C}(\hat{H}^{\text{bond}}) \subset \mathcal{C}(\hat{G}^{\text{site}})$.
	Consequently,
	\begin{align}
		\theta^{\text{bond}}(p, W, L) = \PP{\norms{\mathcal{C}(\hat{G}^{\text{bond}})} >0 }
		 &
		= \PP{\norms{\mathcal{C}(\hat{H}^{\text{bond}})} >0} \\
		 &
		\leq
		\PP{\norms{\mathcal{C}( \hat{G}^{\text{site}})} >0} = \theta^{\text{site}}(p^W, W, L).
	\end{align}
\end{proof}

\begin{proof}[Proof of \zcref{thm:4:main}]
	Define the events $E^{(n)}_t := \sets{\norms{\mathcal{C}(G(F( \cdot , f^{(n)}_t) ))}> 0}$ for $t,n\in\NN$.
	By \zcref[S]{ass:4:zero_updates},
	\begin{equation}
		f_t^{(n)}\odot g^{(n)}(f_t^{(n)} \odot w, \xi_t^{(n)}) = \1{E_t^{(n)}}f_t^{(n)}\odot g^{(n)}(f_t^{(n)} \odot w, \xi^{(n)}_t)
	\end{equation}
	for all $w\in\CW^{(n)}$.
	It then follows that there exists an $M>0$ such that
	\begin{align}
		\EE{\normRs{f_t^{(n)}\odot g^{(n)}(f_t^{(n)} \odot w^{(n)}_t, \xi_t^{(n)})}}\;\;
		 &
		\overset{\mathclap{A~\ref{ass:4:zero_updates}}}{=}\; \;\EE{\1{E_t^{(n)}} \normRs{f^{(n)}_t\odot g^{(n)}(f^{(n)}_t \odot w^{(n)}_t, \xi^{(n)}_t)}}                   \\
		 &
		\overset{\mathclap{(\star)}}{=}\;\;  \EE{\EE{\1{E_t^{(n)}} \normRs{f_t^{(n)}\odot g^{(n)}(f_t^{(n)} \odot w^{(n)}_t, \xi_t^{(n)})}\bigm\vert f^{(n)}_t, w^{(n)}_t}} \\
		 &
		=\;\; \EE{\1{E^{(n)}_t}\EE{ \normRs{f^{(n)}_t\odot g^{(n)}(f^{(n)}_t \odot w^{(n)}_t, \xi^{(n)}_t)}\bigm\vert f^{(n)}_t, w^{(n)}_t }}                               \\
		 &
		\overset{\mathclap{(\star\star)}}{\leq}\;\; \EE{\1{E^{(n)}_t}\EE{ \normRs{g^{(n)}(f^{(n)}_t \odot w^{(n)}_t, \xi^{(n)}_t)}\bigm\vert f^{(n)}_t, w^{(n)}_t}}         \\
		 &
		\overset{\mathclap{A~\ref{ass:4:second_moment}}}{\leq}\;\; \EE{\1{E^{(n)}_t}
		M}                                                                                                                                                                  \\
		 &
		=\; M \theta(n),
		\label{eq:4:proof_E_bound}
	\end{align}
	where $(\star)$ invokes the tower property, and $(\star\star)$ uses that the filters are $\set{0,1}$-valued.
	By the recursive definition of dropout \gls{SGD} and nonnegativity of the $\alpha_t^{(n)}$, it follows that
	\begin{align}
		\lVert{w_{T(n)}^{(n)} - w_0^{(n)}}\rVert\;
		 &
		=\; \normR{\sum_{t=0}^{T(n)-1} \alpha_t^{(n)}f^{(n)}_t \odot g^{(n)}(f^{(n)}_t \odot w^{(n)}_t, \xi_t^{(n)})} \\
		 &
		\overset{\Delta}{\leq}\; \sum_{t=0}^{T(n)-1} \alpha_t^{(n)} \normRs{f^{(n)}_t \odot g^{(n)}(f^{(n)}_t \odot w^{(n)}_t, \xi_t^{(n)})},
	\end{align}
	almost surely.
	Taking expectations of the above, by~\eqref{eq:4:proof_E_bound} we know that there exists an $M>0$ such that,
	\begin{align}
		\EE{\lVert{w_{T(n)}^{(n)} - w_0^{(n)}}\rVert}\;
		 &
		\leq
		\sum_{t=0}^{T(n)-1} \alpha_t^{(n)} \EE{\normRs{f^{(n)}_t \odot g^{(n)}(f^{(n)}_t \odot w^{(n)}_t, \xi_t^{(n)})}} \\
		 &
		\overset{\mathclap{\eqref{eq:4:proof_E_bound}}}{\leq}
		\; M \theta(n) \sum_{t=0}^{T(n)-1} \alpha_t^{(n)}.
	\end{align}
\end{proof}

\begin{proof}[Proof of \zcref{cor:4:constantW}]
	The proof of the first statement follows directly from \zcref[S]{thm:4:main} when bounding $\theta(n)$ using \zcref[S]{prop:3:bond:1}:
	\begin{equation}
		\theta(n) \leq {(1-p^{{W(n)}^2})}^{n+1} = \exp((n+1)\log(1-p^{{W(n)}^2})) \leq \exp(-(n+1)p^{{W(n)}^2})\leq \exp(-np^{{W(n)}^2}).
	\end{equation}
	We prove both statements (i) and (ii) using properties of (generalized) harmonic numbers.
	There namely exists a $K_1>0$ such that
	\begin{equation}
		\label{eq:4:integral1}
		\sum_{t=0}^{T(n)-1}
		\frac{\alpha}{{(t+1)}^\rho}
		=
		\alpha H_{T(n)}^{(\rho)}
		\leq
		K_1 {T(n)}^{1-\rho}
	\end{equation}
	and likewise there exists a $K_2>0$ such that
	\begin{equation}
		\label{eq:4:integral2}
		\sum_{t=0}^{T(n)-1}
		\frac{\alpha}{t+1}
		=
		\alpha H_{T(n)}
		\leq
		K_2 \log T(n)
		.
	\end{equation}

	To establish statement (i), observe that from~\eqref{eq:4:integral1} it follows that there exists a $K_1>0$ such that
	\begin{align}
		\exp(-np^{W^2})\sum_{t=0}^{T(n)-1}\frac{\alpha}{{(t+1)}^{\rho}}
		 &
		\leq
		K_1\exp(-np^{{W(n)}^2}){T(n)}^{1-\rho} \\
		 &
		= O\left(\exp(-(1-c(1-\rho))np^{{W(n)}^2})\right)
	\end{align}
	which tends to $0$ because $c<{(1-\rho)}^{-1}$.
	Likewise, to establish (ii), it follows from~\eqref{eq:4:integral2} that there exists a $K_2>0$ such that
	\begin{align}
		\exp(-np^{W^2})\sum_{t=0}^{T(n)-1}\frac{\alpha}{t+1}
		 &
		\leq
		K_2\exp{(-np^{{W(n)}^2})}\log(T(n)) \\
		 &
		= O\left(\exp(-(1-c)np^{W^2})\right)
	\end{align}
	which tends to $0$ because $c<1$.
\end{proof}

\begin{proof}[Proof of \zcref{lemma:4:no_path}]
	Let $w\in\CW$.
	For $\ell = 1, \dots, L+1$, let $F_\ell( \cdot , w) := \sigma_\ell \circ S_\ell \circ \dots \circ \sigma_1 \circ S_1( \cdot)$ be the \gls{NN} up to layer $\ell$.
	This satisfies the recursive relationship $F_\ell( \cdot , w) = \sigma_\ell \circ S_\ell \circ F_{\ell-1}( \cdot , w)$ and note that $F_{L+1} = F$.
	Define $\CC_\ell := \CC(G(F_\ell( \cdot , w)))$.
	These are the vertices in layer $\ell$ connected to the input layer.
	Let $\bar{\CC_\ell}$ denote the set of vertices in layer $\ell$ that \emph{do not} have paths from the input to the vertex.

	We now prove the statement by induction.
	For $\ell = 1$, $F_1(x, w) = \sigma_1( A_1x + b_1)$, if $i \in \overline{\CC_1}$ then ${[A_1]}_{i,j} = 0$ for all $j=1, \dots, W_0$, implying that ${[A_1 x]}_i = 0$ and thus ${[F_1(x, w)]}_i = {[\sigma_1(A_1 x + b_1)]}_i = \sigma_1({[A_1 x + b_1]}_i) = \sigma_1({[b_1]}_i) = {[F_1(0, w)]}_i$.
	Observe that $\sigma$ is applied element-wise and therefore by abuse of notation, ${[\sigma(x)]}_i = \sigma(x_i)$.
	This yields the base case.

	We now assume that ${[F_\ell(x, w)]}_i = {[F_\ell(0, w)]}_i$ for each node $i \in \overline{\CC_\ell}$.
	What remains is to prove the inductive step.
	Let $k \in \overline{\CC_{\ell+1}}$, then $k$ cannot be connected to any vertex in $\CC_{\ell}$, else there would be a path to $k$, therefore ${[A_{\ell + 1}]}_{k,i}= 0$ for all $i \in \CC_{\ell}$.
	Therefore, ${[A_{\ell + 1}F_\ell(x, w)]}_k$ is a linear combination of ${[F_\ell(x, w)]}_i$ for $i\in \overline{\CC_{\ell}}$ and so ${[A_{\ell + 1}F_\ell(x, w)]}_k = {[A_{\ell + 1}F_\ell(0, w)]}_k$ by our assumption.
	As such, it follows that for all $k\in \overline{\CC_{\ell+1}}$
	\begin{equation}
		\begin{split}{[F_{\ell + 1}(x, w)]}_k
			 &
			=
			{[\sigma_{\ell+1}(A_{\ell+1} F_\ell(x, w) + b_{\ell+1})]}_k
			=
			\sigma_{\ell+1}({[A_{\ell+1} F_\ell(x, w) + b_{\ell+1}]}_k) \\
			 &
			= \sigma_{\ell+1}({[A_{\ell+1} F_\ell(0, w) + b_{\ell+1}]}_k) =
			{[F_{\ell+1}(0, w)]}_k.
		\end{split}
	\end{equation}
	By our induction argument, we see that this holds iteratively up to layer $L+1$.
	We know that $\CC_{L+1} = \emptyset$ and therefore ${[F_{L+1}(x, w)]}_i = {[F_{L+1}(0, w)]}_i$ for all $i=1, \dots, W$.
	We thus conclude that $F(x, w) = F_{L+1}(0, w)$ for all $x\in\CX$.
\end{proof}

\begin{proof}[Proof of \zcref{lemma:4:no_bias_works}]
	Let $w\in\CW$.
	Observe that, by \assref{ass:4:int:no_bias}, $F(0, w) = 0$ for all $w\in\CW$ because $\sigma_\ell(0)=0$ and the \gls{NN} $F$ does not have biases.
	Let $\omega \in \Omega$ such that $\CC(G(F(\,\cdot\, , f(\omega)))) = \emptyset$, then by Lemma~\ref{lemma:4:no_path} $F(x, f(\omega)\odot w) = F(0, f(\omega)\odot w) = 0$.
	Consequently, $\ell(F(X_{t, i}, f(\omega)\odot w), Y_{t, i}) = \ell(0, Y_{t, i})$, which is independent of $w$ and so $\nabla \ell(F(X_{t, i}, f(\omega)\odot w), Y_{t, i}) = \nabla \ell(0, Y_{t, i}) = 0$ for all $i \in\NN$, $t=1,\ldots, m$ and thus $g(f(\omega)\odot w, \xi_t) = 0 $ for all $w\in \CW,\, t=1,\dots,m$.
\end{proof}

\end{document}